\documentclass[11pt]{article}

\usepackage[margin=1in]{geometry}

\usepackage{amsmath,amssymb,amsfonts,amsthm}
\usepackage{bm}
\usepackage{bbm}
\usepackage{mathrsfs}

\usepackage{graphicx}
\usepackage{grffile}
\usepackage{booktabs}
\usepackage{subcaption}

\usepackage{authblk}

\usepackage{algorithm}
\usepackage{algpseudocode}

\usepackage[dvipsnames]{xcolor}
\usepackage{enumerate}

\usepackage[round,authoryear]{natbib}

\setcitestyle{aysep={,},yysep={;},notesep={, }}

\usepackage[colorlinks=true,
            linkcolor=MidnightBlue,
            citecolor=MidnightBlue,
            urlcolor=MidnightBlue]{hyperref}

\theoremstyle{plain}

\newtheorem{theorem}{Theorem}[section]
\newtheorem{lemma}[theorem]{Lemma}
\newtheorem{proposition}[theorem]{Proposition}
\newtheorem{corollary}[theorem]{Corollary}
\theoremstyle{definition}
\newtheorem{definition}[theorem]{Definition}
\newtheorem{example}[theorem]{Example}
\newtheorem*{remark}{Remark}

\newcommand{\be}{\begin{equation}}
\newcommand{\ee}{\end{equation}}
\newcommand{\ben}{\begin{enumerate}}
\newcommand{\een}{\end{enumerate}}
\newcommand{\bea}{\begin{eqnarray}}
\newcommand{\eea}{\end{eqnarray}}
\newcommand{\bean}{\begin{eqnarray*}}
\newcommand{\eean}{\end{eqnarray*}}

\newcommand{\one}[1]{\mathbf{1}_{\{#1\}}}

\newcommand{\E}{\mathbb{E}}

\DeclareMathOperator{\diag}{diag}

\begin{document}

\title{\textbf{Exact Recovery in the Data Block Model}}

\author[1]{Amir R.~Asadi\thanks{asadi@statslab.cam.ac.uk}}
\author[2]{Akbar Davoodi}
\author[3]{Ramin Javadi\thanks{rjavadi@iut.ac.ir}}
\author[4]{Farzad Parvaresh\thanks{f.parvaresh@eng.ui.ac.ir}}

\affil[1]{Statistical Laboratory, Centre for Mathematical Sciences, University of Cambridge, Cambridge, United Kingdom}
\affil[3]{Department of Mathematical Sciences, Isfahan University of Technology, Isfahan, Iran}
\affil[4]{Department of Electrical Engineering, Faculty of Engineering, University of Isfahan, Isfahan, Iran}

\date{February 5, 2026}

\maketitle

\begin{abstract}
Community detection in networks is a fundamental problem in machine learning and statistical inference, with applications in social networks, biological systems, and communication networks. The stochastic block model (SBM) serves as a canonical framework for studying community structure, and {exact recovery}, identifying the true communities with high probability, is a central theoretical question. While classical results characterize the phase transition for exact recovery based solely on graph connectivity, many real-world networks contain additional data, such as node attributes or labels. In this work, we study exact recovery in the Data Block Model (DBM), an SBM augmented with node-associated data, as formalized by Asadi, Abbe, and Verd\'{u} (2017). We introduce the Chernoff--TV divergence and use it to characterize a sharp exact recovery threshold for the DBM. We further provide an efficient algorithm that achieves this threshold, along with a matching converse result showing impossibility below the threshold. 
Finally, simulations validate our findings and demonstrate the benefits of incorporating vertex data as side information in community detection.
\end{abstract}

\section{Introduction}
\subsection{Motivation and Context}
Community detection and clustering are fundamental problems in machine learning with applications in various domains. The goal is to uncover hidden community structures within a network by analyzing patterns of interactions. Many real-world datasets can be represented as networks with community structures, including social networks, biological networks, and collaboration networks \citep{girvan2002community, backstrom2006group}. In these settings, communities naturally partition data into groups of nodes that share common properties, where nodes with similar characteristics are more likely to belong to the same group.

Community detection supports numerous tasks, including recommendation systems \citep{gasparetti2021community}, fraud detection \citep{sarma2020bank}, and biological network analysis \citep{rahiminejad2019topological}, and is often treated as an unsupervised learning problem \citep{abbe017}.
Since the 1980s, research in this area has expanded significantly, leading to a wide range of models and algorithms developed within machine learning, computer science, network science, social science, and statistical physics.

Block models are a family of random graphs with planted clusters \citep{goldenberg2010survey}. The main canonical model is the stochastic block model (SBM), a generative model for random graphs with planted community structure and a widely studied framework for community detection. SBMs provide a theoretical foundation for studying community detection by modeling graphs where nodes belong to latent communities and edges are generated based on intra- and inter-community connection probabilities. Typically, the probability of an edge existing between two nodes depends on their community membership, with edges often being denser within communities than between them. First introduced by \cite{holland1983stochastic} in the context of social network analysis, the SBM has since become a fundamental tool in statistics, machine learning, and network science. It serves as an important benchmark for evaluating community detection algorithms, as it provides a ground truth for the community structure, enabling rigorous theoretical analysis. Despite its simplicity, the SBM captures key bottlenecks in community detection and has inspired powerful algorithms. Similar to the discrete memoryless channel in information theory, the SBM serves as a tractable yet insightful model, guiding the development of methods that extend to more complex real-world graphs. Its widespread adoption in different fields is evidenced by the variety of names and interpretations it has acquired: the planted partition model in theoretical computer science \citep{mcsherry2001spectral}, the inhomogeneous random graph in mathematics \citep{bollobas2007phase}, the planted spin-glass model in physics \citep{decelle2011asymptotic}, and the low-rank random matrix model in random matrix theory \citep{mcsherry2001spectral}.
These diverse perspectives highlight the versatility of the model and its central role in community detection research.

Traditional SBMs primarily focus on network connectivity, yet real-world networks often come enriched with additional information, such as node attributes, labels, or features. For instance, on social media networks, individuals not only form friendships but also generate posts and display characteristics like location, profession, or interests. Similarly, citation networks classify papers by research topics, and biological networks group proteins based on functional properties. In these examples, community structure emerges not only from connectivity but also from these shared external features. Thus, community effects influence both the network's structure and the underlying data distributions associated with its nodes. This interplay between network topology and node-specific data suggests that incorporating vertex data can significantly enhance community detection. 

A natural way to model this setting is through the \emph{Data Block Model} (DBM), an extension of the SBM that integrates additional data as node attributes. First introduced by \citet{abbe2016graph} and later formalized by \citet{asadi2017compressing}, the DBM assumes that nodes belonging to the same community share a common prior distribution for their attributes. \citet{asadi2017compressing} initiated the study of how such vertex data influences the exact-recovery threshold by extending the CH-divergence framework of \citet{abbe015} using Chernoff-information-based arguments. The DBM provides a principled foundation for analyzing the role of side information in exact recovery, connecting network-based and data-driven approaches to clustering. Leveraging node-specific data as side information can be essential for exact recovery, particularly in regimes where the graph structure alone is insufficient. By incorporating both connectivity patterns and node features, the DBM enables a more refined understanding of the conditions under which communities can be correctly identified. We next review further related work along these lines.

\subsection{Further Related Work}
A substantial body of literature studies the fundamental limits and algorithms for community detection in SBMs. For the two-community symmetric SBM, exact recovery is impossible in the constant-degree regime but exhibits a sharp phase transition in the logarithmic-degree regime; moreover, this threshold is efficiently achievable \citep{ABH16,mossel2016consistency}. For the general SBM, Abbe and Sandon introduced the CH-divergence and established a necessary and sufficient threshold for exact recovery together with near-linear-time algorithms that attain it \citep{abbe015,abbe2015recovering}; see also the survey \citep{abbe017}.

Side information attached to vertices has been explored in several complementary models. For the balanced binary SBM, \citet{SN18} analyzed noisy labels (error probability $\alpha$) and partially revealed labels (known on a fraction $1-\epsilon$ of nodes), identifying when and by how much side information tightens the exact-recovery threshold. This line was extended to continuous-valued side information in \citet{SN19}, closing the gap between necessary and sufficient conditions in \citet{SN18}. Sharp thresholds have also been obtained when each node observes $O(\log n)$ i.i.d.\ attribute samples \citep{SZH21}, along with an efficient semidefinite programming (SDP) algorithm that achieves the limit. Consistently, \citet{ESN21} showed that an SDP algorithm can match the maximum-likelihood exact-recovery threshold across several side-information models (noisy labels, partially revealed labels, and multiple features per node).

A closely related line of work studies \emph{contextual stochastic block models}, in which node covariates are observed alongside the graph and jointly inform community detection. From a modeling viewpoint, contextual SBMs can be regarded as special cases of the data block model, obtained by specifying particular families of node-attribute distributions and their dependence on the latent community labels. \cite{deshpande2018contextual} introduced a contextual SBM in which both edge formation and node features are driven by the underlying communities, and identified regimes where combining graph and covariate information enables recovery beyond what is possible from the graph alone. Building on this framework, \cite{braun2022iterative} proposed an iterative clustering algorithm for contextual SBMs and established optimality guarantees under suitable assumptions on the graph sparsity and feature distributions. Together, these works highlight that side information in the form of node features can substantially affect both recovery thresholds and algorithmic performance.

Related node-attributed SBM models have also been studied from an information-theoretic perspective. In particular, \citet{dreveton2023exactAttributedSBM} investigate exact recovery for node-attributed SBMs and characterize a phase transition based on a divergence criterion, together with a Bregman hard clustering viewpoint. In this paper, we revisit this regime in the logarithmic-degree setting and provide a refined characterization based on the Chernoff--TV divergence. In this restricted setting, made precise in Subsection~\ref{subsec:DBM_threshold}, we show that the condition in \citep{dreveton2023exactAttributedSBM} is not necessary for exact recovery as stated, and we derive a corrected threshold.

Beyond these settings, other works examine different forms of auxiliary information and inference goals. Label propagation under partially revealed labels has been analyzed in a general SBM \citep{SSDGMP20}. In \citet{EN20a}, a binary SBM with two latent label variables is considered, where edge probabilities depend on both latent factors; SDP-based guarantees for recovery are provided. In \citet{YLS18}, the symmetric SBM with $k$ communities is considered in which each node is endowed with a binary $m$-dimensional vector. Here, the side information is a noisy and partially observed version of the corresponding vectors. The task is to recover $k$ communities as well as the feature vectors of the nodes, and the paper characterizes the information-theoretic and computational limits of the problem. Hypergraph analogues of these problems have also been developed. For uniform hypergraph SBMs, \citet{ZT21} established exact-recovery thresholds (via a generalized CH-type criterion), while \citet{ZZY21} derived sharp bounds when node labels are noisy or only partially revealed.

\subsection{Our Contributions}
The main contributions of this work are as follows:

\begin{enumerate}
    \item We introduce the Chernoff--TV divergence, using it to give a sharp characterization of the exact recovery threshold in the Data Block Model (DBM). The criterion applies to an arbitrary number of communities, recovers the classical SBM threshold when attributes are uninformative, and reduces to a data-only limit when the graph is uninformative.
    \item We design a polynomial-time algorithm that jointly leverages edge and node data to achieve this Chernoff--TV threshold, establishing computational feasibility at the information-theoretic limit.
    \item We prove a matching converse: below the Chernoff--TV threshold, exact recovery is information-theoretically impossible, even with unbounded computation.
    \item We clarify the relationship between our threshold and existing criteria for exact recovery in node-attributed SBMs, and in the setting of Subsection~\ref{subsec:DBM_threshold} we provide a corrected statement of a previously claimed necessity direction in \citep{dreveton2023exactAttributedSBM}.
    \item We validate the theory with experiments on synthetic datasets, which exhibit sharp empirical phase transitions and consistent gains from incorporating node attributes.
\end{enumerate}

\subsection{Paper Organization}
This paper is structured as follows: In Section~\ref{sec:model} we introduce the Data Block Model and the basic definitions. Section~\ref{sec:prelim} collects preliminary results, including the CH-divergence and related tools. In Section~\ref{sec:main} we define the Chernoff--TV divergence, state our sharp threshold for exact recovery in the Data Block Model, and present a two-stage recovery algorithm together with its achievability and converse analysis. In Section~\ref{sec:simulation}, we explore the experimental setup, present the results, and discuss their implications. Finally, Section~\ref{sec:conclusion} summarizes the key findings of the paper and suggests avenues for future research.
\subsection{Notation} 
We use the following notation throughout the paper. For any positive integer $k$, we write $[k]$ to denote the set $\{1, 2, \ldots, k\}$. Given a matrix $A$, we denote its $r$-th column by $A_r$. The notation $\mathcal{P}_{\bm{\mu}}$ refers to the multivariate Poisson distribution with independent entries and mean vector $\bm{\mu}$. For a vector $v \in \mathbb{R}^n$, we write $\diag(v)$ to denote the $n \times n$ diagonal matrix with entries of $v$ on its diagonal. We use $\mathbb{R}_+$ and $\mathbb{Z}_+$ to denote the set of nonnegative real numbers and nonnegative integers, respectively. $P\ll Q$ indicates that the probability distribution $P$ is absolutely continuous with respect to $Q$.

\section{Models and Definitions}\label{sec:model} In this section, we introduce the models and provide some basic definitions.
\begin{definition}
    Let $n\in\mathbb{N}$ be the number of vertices and $k\in\mathbb{N}$ the number of communities.  Fix
\(
P= (p_1,\dots,p_k),
\) {a probability mass function on } $[k]$ called the prior distribution
and let 
\[
\mathbf{W} = [W_{ab}]_{1\le a,b\le k}\in[0,1]^{k\times k}
\]
be a symmetric matrix of edge-probabilities.

{The stochastic block model} $(G_n,X^n)$ is defined as follows. The community labels
  $X^n = (X_1,\dots,X_n)$ are i.i.d.\ random variables such that 
  \[
    \mathbb{P}(X_i = a) = p_a,\quad a\in[k].
  \]
  The random graph
  $G_n = (Y_{ij})_{1\le i<j\le n}$ is defined on the vertex set 
  $\mathcal V_n = \{v_1,\dots,v_n\}$, where each
  $Y_{ij}\in\{0,1\}$ indicates the absence ($0$) or presence ($1$) of an edge between vertices $v_i$ and $v_j$.
  Conditioned on $X^n$, the edge variables $\{Y_{ij}\}$ are independent and satisfy
  \[
    \mathbb{P}\bigl(Y_{ij}=1 \mid X_i = a, X_j = b\bigr)
    =
    W_{ab}, 
    \quad
    1\le i<j\le n,\ \ a,b\in[k].
  \]
  We write 
\[(G_n,X^n)\sim \mathrm{SBM}(n,k,P,\mathbf{W})\] 
to indicate that $(G_n,X^n)$ follows an SBM with $n$ vertices, $k$ communities, prior distribution $P$, and edge-probability matrix $\mathbf{W}$.
\end{definition}
In this paper, we focus on the logarithmic degree regime, in which the expected degree is $\Theta(\log n)$. Accordingly, we assume that for each $n$ the edge-probability matrix satisfies $\mathbf{W}^{(n)}=({\log n}/{n})\,\mathbf{Q}$ for some fixed symmetric matrix $\mathbf{Q}\in \mathbb{R}_+^{k\times k}$.

\begin{definition}
{The data block model (DBM)} is defined by the triple $(G_n,X^n,U^n)$. Here, $(G_n, X^n)$ is a stochastic block model, referred to as the \emph{underlying SBM}, and  \[U^n:=(U^{(1)},U^{(2)},...,U^{(n)})\] denotes the collection of data files attached to the vertices, where $U^{(i)}$ is the data file located at vertex $i$, for $i=1,2,...,n$. The distribution of $U^{(i)}$ depends on the community to which vertex $i$ belongs. 

We assume that $(G_n,X^n,U^n)$ satisfies the Markov chain $G_n\leftrightarrow X^n\leftrightarrow U^n$, meaning that, conditional on the community labels $X^n$, the graph $G_n$ and the data $U^n$ are independent. Moreover, conditional on $X^n$, the data files $\{U^{(i)}\}_{i=1}^n$ are independent and satisfy $U^{(i)}\mid\{X_i=x\}\sim P_{U|X}(\cdot\mid x)$ for $x\in[k]$. The joint distribution of $(G_n,X^n,U^n)$ is therefore given by  

\begin{align*}
	&\mathbb{P}( G_n=g_n, X^n=x^n, U^n=u^n)\\
    &=\mathbb{P}\bigl(X^n=x^n\bigr) \mathbb{P}\bigl(G_n=g_n\mid X^n=x^n\bigr) \mathbb{P}\bigl(U^n=u^n\mid X^n=x^n\bigr)\\
	&=\prod_{i=1}^nP(x_i)\prod_{1\leq i<j\leq n}\bigl((W_{x_ix_j})^{y_{ij}}(1-W_{x_ix_j})^{1-y_{ij}}\bigr) \prod_{i=1}^nP_{U|X}\bigl(u^{(i)}\mid x_i\bigr).
\end{align*}
We refer to $P_{U\mid X}$ as the \emph{data--community random transformation}. We write
\[(G_n,X^n, U^n)\sim \mathrm{DBM}(n,k, P ,\mathbf{W},P_{U|X})\]
to indicate that $(G_n,X^n,U^n)$ follows a DBM with $n$ vertices, $k$ communities, prior distribution $P$, edge--probability matrix $\mathbf{W}$ and data--community random transformation $P_{U|X}$.
\end{definition}

Consider the following example of a DBM:
\begin{example}\normalfont
	Figure~\ref{fig:DBM} illustrates a realization of a data block model with two communities, corresponding to individuals living in North America and Europe.
Each vertex represents a person, each edge indicates a friendship between two individuals, and each vertex is associated with a data file recording that person's favorite sport.

The model exhibits structure in both the graph and the data.
The distribution of sports differs across the two communities, and the graph displays a higher density of edges within communities than between them.
This example highlights how community membership influences both connectivity patterns and node-level attributes in a DBM.
\begin{figure}
	\centering
	\includegraphics[width=1\textwidth]{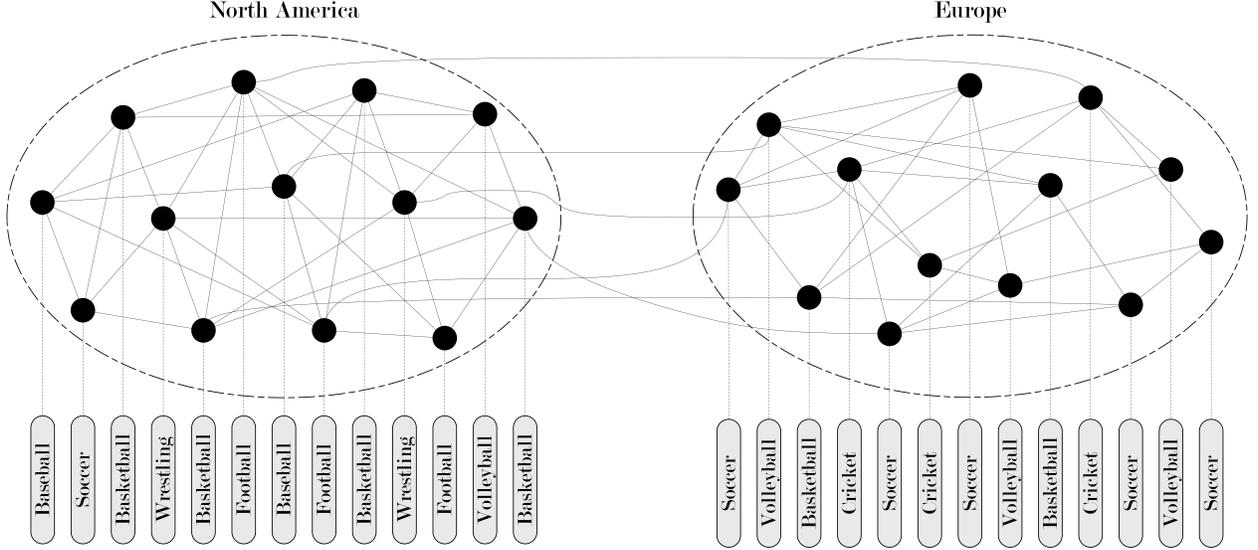}
	\caption{Example of a DBM}\label{fig:DBM}
\end{figure}

\end{example}

We define the solvability of almost exact recovery and exact recovery for a sequence of SBMs as follows:

\begin{definition}
	Let $P$ be a prior distribution on $[k]$ and $\{\mathbf{W}^{(n)}\}_{n=1}^{\infty}$ be a sequence of $k\times k$ edge probability matrices. For each $n\geq 1$, let $(G_n, X^n)\sim \mathrm{SBM}(n,k,P,\mathbf{W}^{(n)})$. The following recoveries are solvable for the sequence $\{(G_n, X^n)\}_{n=1}^{\infty}$ if for each $n\geq 1$ there exists an algorithm which takes $G_n$ as input and outputs $\hat{X}^n=\hat{X}^n(G_n)$ such that
	\begin{itemize}
		\item Almost exact recovery: 
		\[
		\mathbb{P}\left[\max_{\pi\in \text{Sym}(k)}\frac{1}{n}\sum_{i=1}^n \one{\hat{X}_i=\pi(X_i)} \ge 1-o(1)\right]=1-o(1).
		\]
		\item Exact recovery:  
		\[
		\mathbb{P}\left[\exists\,\pi\in \text{Sym}(k):\ \hat{X}^n=\pi(X^n)\right]=1-o(1).
		\]
	\end{itemize}
	Here, $\text{Sym}(k)$ denotes the symmetric group on $[k]$ and $\pi(X^n):=(\pi(X_1),\dots,\pi(X_n))$.
\end{definition}

Similarly, we define solvability of almost exact recovery and exact recovery for a sequence of DBMs as follows:

\begin{definition}\label{def:ER_almostER}
	Let $P$ be a prior distribution on $[k]$, $\{\mathbf{W}^{(n)}\}_{n=1}^{\infty}$ be a sequence of $k\times k$ edge probability matrices, and let $\{P^{(n)}_{U|X}\}_{n=1}^{\infty}$ be a sequence of data--community random transformations. For each $n\geq 1$, let $(G_n, X^n, U^n)\sim \mathrm{DBM}(n,k,P,\mathbf{W}^{(n)}, P^{(n)}_{U|X})$. The following recoveries are solvable for the sequence $\{(G_n, X^n, U^n)\}_{n=1}^{\infty}$ if for each $n\geq 1$ there exists an algorithm which takes $G_n$ and $U^n$ as input and outputs $\hat{X}^n=\hat{X}^n(G_n, U^n)$ such that
	\begin{itemize}
		\item Almost exact recovery: 
		\[
		\mathbb{P}\left[\max_{\pi\in \text{Sym}(k)}\frac{1}{n}\sum_{i=1}^n \one{\hat{X}_i=\pi(X_i)} \ge 1-o(1)\right]=1-o(1).
		\]
		\item Exact recovery:  
		\[
		\mathbb{P}\left[\exists\,\pi\in \text{Sym}(k):\ \hat{X}^n=\pi(X^n)\right]=1-o(1).
		\]
	\end{itemize}
	Here, $\text{Sym}(k)$ denotes the symmetric group on $[k]$ and $\pi(X^n):=(\pi(X_1),\dots,\pi(X_n))$.
\end{definition}

\section{Preliminaries}\label{sec:prelim}

This section introduces the main information-theoretic quantities that underpin our analysis.
We begin with the Chernoff--Hellinger (CH) divergence, which plays a central role in characterizing exact recovery thresholds for stochastic block models in the logarithmic-degree regime.

\begin{definition}[Chernoff--Hellinger divergence]
Let $R$ and $S$ be positive measures defined on a finite set $\mathcal{X}$ (not necessarily probability measures).
The Chernoff--Hellinger (CH) divergence between $R$ and $S$ is defined as
\begin{equation*}\label{CH}
	D_+(R\|S)
	:= \max_{t\in[0,1]} \sum_{x\in \mathcal{X}} f_t\!\left(\frac{R(x)}{S(x)}\right) S(x),
\end{equation*}
where $f_t(y) := 1 - t + ty - y^t$.
\end{definition}

The terminology \emph{Chernoff--Hellinger divergence} was introduced in \citet{abbe015} to emphasize its connection to both the Chernoff information and the Hellinger divergence.
Its importance stems from the fact that it provides an explicit and sharp characterization of the exact recovery threshold for SBMs in the logarithmic-degree regime, as formalized by the following theorem.

\begin{theorem}[\cite{abbe015}]\label{AS15}
Let $P$ be a prior distribution on $[k]$, and let $\mathbf{Q}\in \mathbb{R}_+^{k\times k}$ be a symmetric matrix.
For each $n\ge 1$, let
\[
(G_n,X^n)\sim \mathrm{SBM}\!\left(n,k,P,\tfrac{\log n}{n}\mathbf{Q}\right).
\]
Exact recovery for the sequence $\{(G_n,X^n)\}_{n\ge1}$ is achievable (and efficiently so) if
\begin{equation*}
	D_+\!\left((\diag(P)\mathbf{Q})_i \,\middle\|\, (\diag(P)\mathbf{Q})_j\right) > 1,
\end{equation*}
for all $i,j\in[k]$ with $i\neq j$.
Conversely, exact recovery is impossible if there exist $i\neq j$ such that
\[
	D_+\!\left((\diag(P)\mathbf{Q})_i \,\middle\|\, (\diag(P)\mathbf{Q})_j\right) < 1,
\]
where $(\diag(P)\mathbf{Q})_i$ denotes the $i$-th column of $\diag(P)\mathbf{Q}$.
\end{theorem}

Theorem~\ref{AS15} shows that, in the logarithmic-degree regime, SBMs exhibit a sharp phase transition for exact recovery.
This transition is governed by the quantity
\[
\min_{i\neq j} D_+\!\left((\diag(P)\mathbf{Q})_i \,\middle\|\, (\diag(P)\mathbf{Q})_j\right).
\]

We next recall an alternative representation of the CH-divergence, derived in \citet{asadi2017compressing}, which connects it to the Chernoff information between multivariate Poisson distributions.
To this end, we first review two standard divergence measures.

\begin{definition}[Relative entropy]
Let $P_X$ and $Q_X$ be discrete probability distributions defined on the same set $\mathcal{A}$.
The relative entropy (or Kullback--Leibler divergence) between $P_X$ and $Q_X$ is defined as
\[
D(P_X\|Q_X)
= \E\!\left[\log\!\left(\frac{P_X(X)}{Q_X(X)}\right)\right],
\]
where $X\sim P_X$, provided that $P_X\ll Q_X$; otherwise, $D(P_X\|Q_X)=\infty$.
\end{definition}

\begin{definition}[Chernoff information {\citep{chernoff1952measure}}]
Let $P_1$ and $P_2$ be discrete probability distributions on a finite or countable set $\mathcal{X}$.
The Chernoff information between $P_1$ and $P_2$ is defined as
\begin{equation}\label{Chernoff}
	C(P_1\|P_2)
	:= -\log\Biggl(\min_{\lambda\in[0,1]}
	\sum_{x\in\mathcal{X}} P_1(x)^{\lambda} P_2(x)^{1-\lambda}\Biggr).
\end{equation}
\end{definition}

The following proposition establishes a key identity between the CH-divergence and Chernoff information for multivariate Poisson distributions with independent coordinates.

\begin{proposition}[CH-divergence as Chernoff information \citep{asadi2017compressing}]\label{CHChernoff}
Let $\mathbf{a},\mathbf{b}\in\mathbb{R}_+^n$.
Then
\[
D_+(\mathbf{a}\|\mathbf{b})
= C\!\left(\mathcal{P}_{\mathbf{a}} \middle\| \mathcal{P}_{\mathbf{b}}\right),
\]
where $\mathcal{P}_{\mathbf{a}}$ denotes the product Poisson distribution with mean vector $\mathbf{a}$.
\end{proposition}

Finally, we recall the definition of total variation distance.

\begin{definition}[Total variation]
Let $P$ and $Q$ be probability distributions on a measurable space $\Omega$.
Their total variation distance is defined as
\[
\mathrm{TV}(P,Q)
:= \sup_{A\subseteq \Omega} |P(A)-Q(A)|.
\]
\end{definition}

\section{Main theoretical results}\label{sec:main}

\subsection{Chernoff--TV divergence}

We begin by introducing a new divergence measure, which we term the \emph{Chernoff--TV divergence}.
This quantity captures, in a unified manner, the joint contribution of graph structure and node attributes to the exact recovery threshold in the data block model.

\begin{definition}[Chernoff--TV divergence]
Let $P_X^{(1)}$ and $P_X^{(2)}$ be discrete probability measures on a finite set $\mathcal{X}$, and let
$Q_U^{(1)}$ and $Q_U^{(2)}$ be discrete probability measures on a finite set $\mathcal{U}$.
The Chernoff--TV divergence between the pairs $\bigl(P_X^{(1)},Q_U^{(1)}\bigr)$ and
$\bigl(P_X^{(2)},Q_U^{(2)}\bigr)$ is defined as
\begin{align*}
&D_{\rm{CT}}\Bigl(P^{(1)}_X,Q^{(1)}_U\Big\|P^{(2)}_X,Q^{(2)}_U\Bigr):=\\&-\log \Biggl(\sum_{u \in \mathcal{U}}
		\min_{\lambda_u \in [0,1]}
    \sum_{x\in \mathcal{X}} 
    \Bigl(P_X^{(1)}(x)Q_U^{(1)}(u)\Bigr)^{\lambda_u} \Bigl(P_X^{(2)}(x)Q_U^{(2)}(u)\Bigr)^{1-\lambda_u}\Biggr)
    . 
\end{align*}
\end{definition}
The next proposition shows that $D_{\mathrm{CT}}$ recovers classical information measures in special cases.

\begin{proposition}\label{prop:CT_special_cases}
Let $P_X^{(1)},P_X^{(2)}$ be pmfs on $\mathcal{X}$ and $Q_U^{(1)},Q_U^{(2)}$ be pmfs on $\mathcal{U}$.

\begin{enumerate}[(i)]
\item If $Q_U^{(1)}=Q_U^{(2)}$, then
\[
D_{\mathrm{CT}}\Bigl(P_X^{(1)},Q_U^{(1)} \,\Big\|\, P_X^{(2)},Q_U^{(2)}\Bigr)
= C\Bigl(P_X^{(1)} \,\Big\|\, P_X^{(2)}\Bigr).
\]

\item If $P_X^{(1)}=P_X^{(2)}$, then
\[
D_{\mathrm{CT}}\Bigl(P_X^{(1)},Q_U^{(1)} \,\Big\|\, P_X^{(2)},Q_U^{(2)}\Bigr)
= -\log \Bigl(1-\mathrm{TV}\Bigl(Q_U^{(1)},Q_U^{(2)}\Bigr)\Bigr).
\]
\end{enumerate}
\end{proposition}
\begin{proof}
(i) If $Q_U^{(1)}=Q_U^{(2)}=:Q_U$, then for any choice of $\{\lambda_u\}_{u\in\mathcal{U}}$,
\[
\sum_{u\in\mathcal{U}}\sum_{x\in\mathcal{X}}
\bigl(P_X^{(1)}(x)Q_U(u)\bigr)^{\lambda_u}
\bigl(P_X^{(2)}(x)Q_U(u)\bigr)^{1-\lambda_u}
=
\sum_{u\in\mathcal{U}} Q_U(u)\sum_{x\in\mathcal{X}} (P_X^{(1)}(x))^{\lambda_u}(P_X^{(2)}(x))^{1-\lambda_u}.
\]
The minimum over $\lambda_u\in[0,1]$ can be taken separately for each $u$, yielding
\[
\min_{\lambda_u\in[0,1]}\sum_{x\in\mathcal{X}} (P_X^{(1)}(x))^{\lambda_u}(P_X^{(2)}(x))^{1-\lambda_u}
=\min_{\lambda\in[0,1]}\sum_{x\in\mathcal{X}} (P_X^{(1)}(x))^{\lambda}(P_X^{(2)}(x))^{1-\lambda}
= C\bigl(P_X^{(1)}\|P_X^{(2)}\bigr),
\]
and therefore the sum over $u$ equals $C(P_X^{(1)}\|P_X^{(2)})$.

(ii) If $P_X^{(1)}=P_X^{(2)}=:P_X$, then for each $u$ and any $\lambda_u\in[0,1]$,
\begin{align*}
	\sum_{x\in\mathcal{X}}
\bigl(P_X(x)Q_U^{(1)}(u)\bigr)^{\lambda_u}
\bigl(P_X(x)Q_U^{(2)}(u)\bigr)^{1-\lambda_u}
&=
\Bigl(Q_U^{(1)}(u)\Bigr)^{\lambda_u}\Bigl(Q_U^{(2)}(u)\Bigr)^{1-\lambda_u}\sum_{x\in\mathcal{X}} P_X(x)\\
&=
\Bigl(Q_U^{(1)}(u)\Bigr)^{\lambda_u}\Bigl(Q_U^{(2)}(u)\Bigr)^{1-\lambda_u}.
\end{align*}
Minimizing over $\lambda_u\in[0,1]$ yields $\min\{Q_U^{(1)}(u),Q_U^{(2)}(u)\}$, hence
\[
\sum_{u\in\mathcal{U}} \min_{\lambda_u\in[0,1]}
\Bigl(Q_U^{(1)}(u)\Bigr)^{\lambda_u}\Bigl(Q_U^{(2)}(u)\Bigr)^{1-\lambda_u}
=
\sum_{u\in\mathcal{U}} \min\{Q_U^{(1)}(u),Q_U^{(2)}(u)\}.
\]
Finally, for discrete distributions,
$\sum_{u}\min\{Q^{(1)}(u),Q^{(2)}(u)\}=1-\mathrm{TV}(Q^{(1)},Q^{(2)})$, yielding the result.
\end{proof}
Proposition~\ref{prop:CT_special_cases} shows that the Chernoff--TV divergence interpolates naturally between Chernoff information and a function of total variation. In our applications, the distribution $P_X$ will correspond to the \emph{graph-based} contribution (arising from the neighborhood statistics induced by the SBM), while $Q_U$ will correspond to the \emph{data} (side-information) contribution carried by node attributes; see the next Subsection for the precise identification.

\begin{lemma}\label{prop:CT_upper_bound}
Let $P_X^{(1)},P_X^{(2)}$ be pmfs on $\mathcal X$ and $Q_U^{(1)},Q_U^{(2)}$ be pmfs on $\mathcal U$.
Then
\[
D_{\mathrm{CT}}\Bigl(P_X^{(1)},Q_U^{(1)}\Big\|P_X^{(2)},Q_U^{(2)}\Bigr)
\le
C\Bigl(P_X^{(1)}\Big\|P_X^{(2)}\Bigr)
-\log\!\Bigl(1-\mathrm{TV}\bigl(Q_U^{(1)},Q_U^{(2)}\bigr)\Bigr).
\]
\end{lemma}

\begin{proof}
Fix $u\in\mathcal U$. For any $\lambda\in[0,1]$,
\begin{align*}
\sum_{x\in\mathcal X}
\bigl(P_X^{(1)}(x)Q_U^{(1)}(u)\bigr)^{\lambda}
\bigl(P_X^{(2)}(x)Q_U^{(2)}(u)\bigr)^{1-\lambda}
&=
\bigl(Q_U^{(1)}(u)\bigr)^{\lambda}\bigl(Q_U^{(2)}(u)\bigr)^{1-\lambda}
\sum_{x\in\mathcal X} P_X^{(1)}(x)^\lambda P_X^{(2)}(x)^{1-\lambda}.
\end{align*}
Taking the minimum over $\lambda$ on both sides yields
\begin{align*}
&\min_{\lambda\in[0,1]}\sum_{x\in\mathcal X}
\bigl(P_X^{(1)}(x)Q_U^{(1)}(u)\bigr)^{\lambda}
\bigl(P_X^{(2)}(x)Q_U^{(2)}(u)\bigr)^{1-\lambda}
\\
&\ge
\Bigl(\min_{\lambda\in[0,1]}\sum_{x\in\mathcal X} P_X^{(1)}(x)^\lambda P_X^{(2)}(x)^{1-\lambda}\Bigr)
\min_{\lambda\in[0,1]}\bigl(Q_U^{(1)}(u)\bigr)^{\lambda}\bigl(Q_U^{(2)}(u)\bigr)^{1-\lambda}\\
&=
\Bigl(\min_{\lambda\in[0,1]}\sum_{x\in\mathcal X} P_X^{(1)}(x)^\lambda P_X^{(2)}(x)^{1-\lambda}\Bigr)
\min \Bigl\{Q_U^{(1)}(u),Q_U^{(2)}(u) \Bigr\}.
\end{align*}
Summing over $u$ gives
\[
\sum_{u\in\mathcal U}\min_{\lambda_u\in[0,1]}\sum_{x\in\mathcal X}
\bigl(P_X^{(1)}(x)Q_U^{(1)}(u)\bigr)^{\lambda_u}
\bigl(P_X^{(2)}(x)Q_U^{(2)}(u)\bigr)^{1-\lambda_u}
\ge
A\sum_{u\in\mathcal{U}} \min\{Q_U^{(1)}(u),Q_U^{(2)}(u)\},
\]
where $A:=\min_{\lambda\in[0,1]}\sum_{x\in\mathcal X} P_X^{(1)}(x)^\lambda P_X^{(2)}(x)^{1-\lambda}=\exp(-C(P_X^{(1)}\|P_X^{(2)}))$. Recall that for discrete distributions,
$\sum_{u}\min\{Q^{(1)}(u),Q^{(2)}(u)\}=1-\mathrm{TV}(Q^{(1)},Q^{(2)}).$
Taking logarithms completes the proof.
\end{proof}
\subsection{Exact recovery threshold for the DBM}\label{subsec:DBM_threshold}

Throughout this section, fix $k\in\mathbb{N}$ and a prior distribution $P=(p_1,\ldots,p_k)$ on $[k]$. 
Let $\mathbf{Q}\in\mathbb{R}_+^{k\times k}$ be symmetric and consider the logarithmic-degree regime
\[
\mathbf{W}^{(n)}=\frac{\log n}{n}\mathbf{Q}.
\]
Let $\{\mathcal{U}^{(n)}\}_{n\ge1}$ be a sequence of finite alphabets and let $\{P^{(n)}_{U|X}\}_{n\ge1}$ be a sequence of conditional pmfs such that,
for each $n$ and each $x\in[k]$, $P^{(n)}_{U|X}(\cdot|x)$ is a probability mass function on $\mathcal{U}^{(n)}$.
For each $r\in[k]$, define
\[
\boldsymbol{\mu}_r := (\diag(P)\mathbf{Q})_r \in \mathbb{R}_+^k,
\qquad 
\boldsymbol{\mu}_r^{(n)} := \boldsymbol{\mu}_r \log n,
\]
where $(\diag(P)\mathbf{Q})_r$ denotes the $r$-th column of $\diag(P)\mathbf{Q}$.
For $s,t\in[k]$, $s\neq t$, define the (normalized) Chernoff--TV separation
\begin{equation}
D_{s,t}
:= \liminf_{n\to\infty}\frac{1}{\log n}\,
D_{\mathrm{CT}}\Bigl(\mathcal{P}_{\boldsymbol{\mu}_s^{(n)}},\, P^{(n)}_{U|X}(\cdot|s)\,\Big\|\, 
\mathcal{P}_{\boldsymbol{\mu}_t^{(n)}},\, P^{(n)}_{U|X}(\cdot|t)\Bigr),
\label{eq:Dst_main}
\end{equation}
where $\mathcal{P}_{\boldsymbol{\mu}_r^{(n)}}$ denotes the product Poisson distribution on $\mathbb{Z}_+^k$ with independent coordinates and mean vector $\boldsymbol{\mu}_r^{(n)}$.

\begin{theorem}[Achievability]\label{thm:main_ach}
Let $(G_n,X^n,U^n)\sim \mathrm{DBM}\!\left(n,k,P,\mathbf{W}^{(n)},P^{(n)}_{U|X}\right)$.
If for all $s,t\in[k],\ s\neq t,$
\begin{equation}
D_{s,t}>1,
\label{eq:ach_condition}
\end{equation}
then exact recovery is achievable, and there exists a polynomial-time (in fact, near-linear-time) algorithm that achieves exact recovery.
\end{theorem}

\begin{theorem}[Converse]\label{thm:main_conv}
Let $(G_n,X^n,U^n)\sim \mathrm{DBM}\!\left(n,k,P,\mathbf{W}^{(n)},P^{(n)}_{U|X}\right)$.
If there exist $s,t\in[k]$ with $s\neq t$ such that
\begin{equation}
D_{s,t}<1,
\label{eq:conv_condition}
\end{equation}
then exact recovery is impossible.
\end{theorem}
The proof of Theorem~\ref{thm:main_ach} is given in Subsection~\ref{sub:suff},
and the proof of Theorem~\ref{thm:main_conv} is given in Subsection~\ref{sub:nec}.
\begin{remark}[Reduction to the classical SBM threshold]\label{rem:reduction_to_SBM}
Theorems \ref{thm:main_ach} and \ref{thm:main_conv} recover the sharp exact-recovery threshold of the standard SBM when no vertex data are available (or when the data carry no information about the labels).
Indeed, suppose the attribute distributions are identical across communities, i.e.,
\[
P^{(n)}_{U|X}(\cdot|s)=P^{(n)}_{U|X}(\cdot|t)\qquad \text{for all } s,t\in[k].
\]
Then, by Proposition~\ref{prop:CT_special_cases}(i),
\[
D_{\mathrm{CT}}\Bigl(\mathcal{P}_{\bm{\mu}_s^{(n)}},\,P^{(n)}_{U|X}(\cdot|s)\,\Big\|\,\mathcal{P}_{\bm{\mu}_t^{(n)}},\,P^{(n)}_{U|X}(\cdot|t)\Bigr)
=
C\Bigl(\mathcal{P}_{\bm{\mu}_s^{(n)}}\Big\| \mathcal{P}_{\bm{\mu}_t^{(n)}}\Bigr).
\]
Moreover, by the Chernoff information representation of the CH-divergence (Proposition \ref{CHChernoff}),
\[
C\Bigl(\mathcal{P}_{\bm{\mu}_s^{(n)}}\Big\| \mathcal{P}_{\bm{\mu}_t^{(n)}}\Bigr)
=
D_+\!\left(\bm{\mu}_s\middle\|\bm{\mu}_t\right).
\]
Therefore the criterion in Theorems~\ref{thm:main_ach}--\ref{thm:main_conv} reduces to
\[
D_+\!\left(\bm{\mu}_s\middle\|\bm{\mu}_t\right) > 1 \quad \text{for all } s\neq t,
\]
which is exactly the Abbe--Sandon sharp threshold for exact recovery in the logarithmic-degree SBM.
\end{remark}

\begin{remark}[Comparison with the sufficient condition of \cite{asadi2017compressing}]\label{rem:CT_vs_productChernoff}
A natural approach used in \citep{asadi2017compressing} is to view the graph-based statistic and the vertex attribute as jointly observed under each community and to compare the resulting \emph{product} distributions $\mathcal{P}_{\bm{\mu}_s^{(n)}}\times\,P^{(n)}_{U|X}(\cdot|s)$
and $\mathcal{P}_{\bm{\mu}_t^{(n)}}\times\,P^{(n)}_{U|X}(\cdot|t)$
via the standard Chernoff information. This yields an explicit \emph{sufficient} condition for exact recovery: if the induced Chernoff separation between these product distributions is large enough (under the appropriate $\log n$ normalization in the logarithmic-degree regime), then exact recovery is achievable.

However, this product-Chernoff criterion is generally \emph{not necessary}.
The reason is that the Chernoff information for product distributions optimizes over a \emph{single} exponent parameter $\lambda\in[0,1]$ shared across all attribute outcomes, whereas the Chernoff--TV divergence $D_{\mathrm{CT}}$ optimizes over a \emph{collection} $\{\lambda_u\}_{u\in\mathcal{U}}$ that may depend on the observed attribute value $u$.
Consequently, $D_{\mathrm{CT}}$ can be strictly larger than the corresponding product-Chernoff exponent, and exact recovery may still be possible even when the sufficient product-Chernoff condition fails.
\end{remark}

A natural question is when node data can compensate for a graph that is below the exact-recovery threshold.
The next theorem shows that if the graph-based separation $D_+(\boldsymbol{\mu}_s\|\boldsymbol{\mu}_t)$ is below $1$ and the data distributions are not sufficiently separated in total variation, then exact recovery remains impossible.

\begin{corollary}[When data cannot help]\label{thm:TV_D+}
Let $(G_n,X^n,U^n)\sim \mathrm{DBM}\!\left(n,k,P,\mathbf{W}^{(n)},P^{(n)}_{U|X}\right)$.
Suppose there exist $s,t\in[k]$, $s\neq t$, such that
\[
D_+\!\left(\boldsymbol{\mu}_s \,\middle\|\, \boldsymbol{\mu}_t\right)<1-\varepsilon
\qquad\text{and}\qquad
\mathrm{TV}\!\left(P^{(n)}_{U|X}(\cdot|s),\,P^{(n)}_{U|X}(\cdot|t)\right)
\le 1-n^{-\varepsilon}
\]
for some $\varepsilon>0$ and all sufficiently large $n$.
Then exact recovery is impossible.
\end{corollary}
\begin{proof}
Let $s\neq t$ satisfy the assumptions. In the logarithmic-degree regime, $\bm{\mu}_s^{(n)}=(\log n)\bm{\mu}_s$
and $\bm{\mu}_t^{(n)}=(\log n)\bm{\mu}_t$, and by the identity
between CH-divergence and Chernoff information for multivariate Poisson
(Proposition~\ref{CHChernoff} together with the scaling
$D_+(c\bm{a}\|c\bm{b})=c\,D_+(\bm{a}\|\bm{b})$), we have
\[
C\Bigl(\mathcal{P}_{\bm{\mu}_s^{(n)}}\Big\|\mathcal{P}_{\bm{\mu}_t^{(n)}}\Bigr)
=
D_+\!\left(\bm{\mu}_s^{(n)}\middle\|\bm{\mu}_t^{(n)}\right)
=
(\log n)\,D_+\!\left(\bm{\mu}_s\middle\|\bm{\mu}_t\right)< (\log n)(1-\varepsilon).
\]
Applying Lemma~\ref{prop:CT_upper_bound} with
\[
P_X^{(1)}=\mathcal{P}_{\bm{\mu}_s^{(n)}},\quad
P_X^{(2)}=\mathcal{P}_{\bm{\mu}_t^{(n)}},\quad
Q_U^{(1)}=P^{(n)}_{U|X}(\cdot|s),\quad
Q_U^{(2)}=P^{(n)}_{U|X}(\cdot|t),
\]
gives, for every $n$,
\begin{align*}
&D_{\mathrm{CT}}\Bigl(\mathcal{P}_{\bm{\mu}_s^{(n)}},\,P^{(n)}_{U|X}(\cdot|s)\,\Big\|\,\mathcal{P}_{\bm{\mu}_t^{(n)}},\,P^{(n)}_{U|X}(\cdot|t)\Bigr) \\
&\qquad\le
C\Bigl(\mathcal{P}_{\bm{\mu}_s^{(n)}}\Big\|\mathcal{P}_{\bm{\mu}_t^{(n)}}\Bigr)
-\log\!\Bigl(1-\mathrm{TV}\bigl(P^{(n)}_{U|X}(\cdot|s),P^{(n)}_{U|X}(\cdot|t)\bigr)\Bigr)\\
&\qquad=(\log n)\,D_+\!\left(\bm{\mu}_s\middle\|\bm{\mu}_t\right)-\log\!\Bigl(1-\mathrm{TV}\bigl(P^{(n)}_{U|X}(\cdot|s),P^{(n)}_{U|X}(\cdot|t)\bigr)\Bigr)\\
&\qquad< (\log n)(1-\varepsilon)+\varepsilon\log n\\
&\qquad = \log n.
\end{align*}
Thus, 
\[
D_{s,t}
:=\liminf_{n\to\infty}\frac{1}{\log n}\,
D_{\mathrm{CT}}\Bigl(\mathcal{P}_{\bm{\mu}_s^{(n)}},\,P^{(n)}_{U|X}(\cdot|s)\,\Big\|\,\mathcal{P}_{\bm{\mu}_t^{(n)}},\,P^{(n)}_{U|X}(\cdot|t)\Bigr)<1.
\]
Therefore, by Theorem \ref{thm:main_conv}, exact recovery is impossible.
\end{proof}

A well-studied form of vertex side information is the \emph{noisy--erasure label} model, where the node attribute takes values in the community alphabet but is corrupted.
Concretely, conditional on $X_i=x$, the observed label $U^{(i)}$ equals $x$ with high probability, and otherwise it is either erased or replaced by another community label.
The binary symmetric case ($k=2$) was analyzed in \citep{SN18}.
As an application of Theorems~\ref{thm:main_ach} and~\ref{thm:main_conv}, we derive a sharp necessary and sufficient condition for exact recovery for general $k\ge 2$ under a polynomially-vanishing noise model.

\begin{corollary}[Noisy-erasure labels]\label{cor:noisy_erasure_general_k}
Consider the data block model
\[
(G_n,X^n,U^n)\sim \mathrm{DBM}\!\left(n,k,P,\frac{\log n}{n}\mathbf{Q},P^{(n)}_{U|X}\right),
\]
where $P=(p_1,\dots,p_k)$ is a prior on $[k]$ and $\mathbf{Q}\in\mathbb{R}_+^{k\times k}$ is symmetric.
Let $\bm{\mu}_s := (\diag(P)\mathbf{Q})_s\in\mathbb{R}_+^k$ and write $\mu_{s_\ell}$ for its $\ell$-th component.

Assume the label alphabet is $\mathcal{U}=\{1,\dots,k,\varepsilon\}$ where $\varepsilon$ denotes an erasure and, for each $x\in[k]$, the data--community channel satisfies
\[
P_{U|X}^{(n)}(u|x)=
\begin{cases}
n^{-d_{u,x}}, & u\in[k],\ u\neq x,\\
n^{-d_{\varepsilon,x}}, & u=\varepsilon,\\
1-\displaystyle\sum_{\substack{u\in[k]\\u\neq x}} n^{-d_{u,x}}-n^{-d_{\varepsilon,x}}, & u=x,
\end{cases}
\]
for some constants $d_{u,x}>0$ (for $u\neq x$) and $d_{\varepsilon,x}>0$.

Then, for any $s\neq t$, the pairwise divergence rate $D_{s,t}$ appearing in Theorems~\ref{thm:main_ach} and~\ref{thm:main_conv}
coincides with the explicit variational expression
\[
D_{s,t}
=
\min_{u\in\mathcal{U}}\ \max_{\lambda_u\in[0,1]}
\Biggl[
d_{u,s}\lambda_u+d_{u,t}(1-\lambda_u)
+\sum_{\ell=1}^k\Bigl(\mu_{s_\ell}\lambda_u+\mu_{t_\ell}(1-\lambda_u)-\mu_{s_\ell}^{\lambda_u}\mu_{t_\ell}^{1-\lambda_u}\Bigr)
\Biggr],
\]
with the convention $d_{x,x}=0$ for $x\in[k]$.

Consequently, exact recovery is solvable if $D_{s,t}>1$ for all $s\neq t$, and it is not solvable if there exist $s\neq t$ such that $D_{s,t}<1$.
\end{corollary}

\begin{proof}
To find the phase transition threshold $D_{s,t}$, we find the Chernoff--TV divergence
\begin{align*}
& D_{\mathrm{CT}}\left(\mathcal{P}_{\bm{\mu}_s^{(n)}} , P^{(n)}_{U|X}(\cdot|s) \middle\|  \mathcal{P}_{\bm{\mu}_t^{(n)}} , P^{(n)}_{U|X}(\cdot|t) \right)= \nonumber \\
&-\log\Biggl(\sum_{u \in \mathcal{U}} 
\exp\Biggl(\log n  \Biggl( \max_{\lambda_u \in [0,1]} 
d_{u,s} \lambda_u + d_{u,t} (1-\lambda_u) +
\sum_{\ell=1}^k \mu_{s_\ell} \lambda_u + \mu_{t_\ell}(1-\lambda_u) 
-\mu_{s_\ell}^{\lambda_u} \mu_{t_\ell}^{1-\lambda_u} \Biggr)\Biggr) \Biggr).
\end{align*} 
Define
\[
\Delta_{s,t} = \min_{u\in\mathcal{U}} \max_{\lambda_u \in [0,1]} \Biggl[
d_{u,s} \lambda_u + d_{u,t} (1-\lambda_u) +
\sum_{\ell=1}^k \mu_{s_\ell} \lambda_u + \mu_{t_\ell}(1-\lambda_u) 
-\mu_{s_\ell}^{\lambda_u} \mu_{t_\ell}^{1-\lambda_u}\Biggr].
\]
Then,
\begin{align*}
 \exp(-\Delta_{s,t}\log n) 
  &\le \exp\left(-D_{\mathrm{CT}}\left(\mathcal{P}_{\bm{\mu}_s^{(n)}} , P^{(n)}_{U|X}(\cdot|s) \middle\|  \mathcal{P}_{\bm{\mu}_t^{(n)}} , P^{(n)}_{U|X}(\cdot|t) \right)\right) \\
  &\le (k+1)\exp(-\Delta_{s,t} \log n).
\end{align*}
Therefore, the phase transition threshold for the noisy label problem
is given by
\[
 D_{s,t}=\liminf_{n \to \infty}  - \frac{1}{\log n}
    \log(\exp(-\Delta_{s,t} \log n)) = \Delta_{s,t}.
\]
\end{proof}

\begin{example}[Erased labels]\label{exm:erase}
Consider the balanced two-community symmetric model with prior $P=(1/2,1/2)$ and logarithmic degree edge probabilities with
\[
\mathbf{Q}=
\begin{pmatrix}
a & b\\
b & a
\end{pmatrix}.
\]
The side-information alphabet is $\mathcal{U}=\{1,2,\varepsilon\}$, where $\varepsilon$ denotes an erasure. Conditioned on $X\in\{1,2\}$, the observed label is correct with probability $1-n^{-\alpha}$ and erased with probability $n^{-\alpha}$:
\[
P_{U|X}^{(n)}(x|x)=1-n^{-\alpha},\qquad
P_{U|X}^{(n)}(\varepsilon|x)=n^{-\alpha},\qquad
P_{U|X}^{(n)}(\bar x|x)=0,\qquad x\in\{1,2\}.
\]
Applying Corollary~\ref{cor:noisy_erasure_general_k} yields
\begin{align*}
D_{1,2}
&=\max_{\lambda\in[0,1]}
\Biggl[
\alpha\lambda+\alpha(1-\lambda)
+\frac{a}{2}\lambda+\frac{b}{2}(1-\lambda)-\Bigl(\frac{a}{2}\Bigr)^\lambda\Bigl(\frac{b}{2}\Bigr)^{1-\lambda}\\
&\hspace{3.1cm}
+\frac{b}{2}\lambda+\frac{a}{2}(1-\lambda)-\Bigl(\frac{b}{2}\Bigr)^\lambda\Bigl(\frac{a}{2}\Bigr)^{1-\lambda}
\Biggr].
\end{align*}
Choosing $\lambda=\tfrac12$ gives
\[
D_{1,2}
=\alpha+\frac{a+b}{2}-\sqrt{ab}
=\alpha+\frac{(\sqrt a-\sqrt b)^2}{2}.
\]
Hence the exact-recovery phase transition coincides with \citep{SN18}.
\end{example}

\begin{example}[Noisy labels]\label{eq:noisy_labels}
Consider the balanced two-community symmetric model with prior $P=(1/2,1/2)$ and logarithmic-degree edge probabilities with
\[
\mathbf{Q}=
\begin{pmatrix}
a & b\\
b & a
\end{pmatrix}.
\]
The label alphabet is $\mathcal{U}=\{1,2\}$. Conditionally on the true community label $X\in\{1,2\}$, the observed label is flipped with probability $n^{-\alpha}$:
\[
P_{U|X}^{(n)}(x|x)=1-n^{-\alpha},\qquad
P_{U|X}^{(n)}(\bar x|x)=n^{-\alpha},\qquad x\in\{1,2\}.
\]
Applying Corollary~\ref{cor:noisy_erasure_general_k} yields
\begin{align*}
D_{1,2}
&=\max_{\lambda\in[0,1]}
\Biggl[
\alpha\lambda
+\frac{a}{2}\lambda+\frac{b}{2}(1-\lambda)-\Bigl(\frac{a}{2}\Bigr)^\lambda\Bigl(\frac{b}{2}\Bigr)^{1-\lambda}
+\frac{b}{2}\lambda+\frac{a}{2}(1-\lambda)-\Bigl(\frac{b}{2}\Bigr)^\lambda\Bigl(\frac{a}{2}\Bigr)^{1-\lambda}
\Biggr]\\
&=\frac{a+b}{2}
+\max_{\lambda\in[0,1]}
\Biggl[
\alpha\lambda
-\Bigl(\frac{a}{2}\Bigr)^\lambda\Bigl(\frac{b}{2}\Bigr)^{1-\lambda}
-\Bigl(\frac{b}{2}\Bigr)^\lambda\Bigl(\frac{a}{2}\Bigr)^{1-\lambda}
\Biggr].
\end{align*}
Let $T:=\log(a/b)$ and $\eta:=\sqrt{\alpha^2+abT^2}$. The maximizer equals
\[
\lambda^\star=
\frac{1}{2}+\frac{1}{2T}\log\!\left(\frac{\eta+\alpha}{\eta-\alpha}\right)
\qquad\text{when }\ \alpha<\frac{T(a-b)}{2},
\]
and otherwise the optimum is attained at the boundary $\lambda^\star=1$. Substituting gives
\[
D_{1,2}=
\begin{cases}
\displaystyle
\frac{\alpha}{2}+\frac{a+b}{2}
-\frac{\eta}{T}
+\frac{\alpha}{2T}\log\!\left(\frac{\eta+\alpha}{\eta-\alpha}\right),
& \alpha<\frac{T(a-b)}{2},\\[1.2ex]
\displaystyle
\alpha,
& \alpha>\frac{T(a-b)}{2}.
\end{cases}
\]
Hence the exact-recovery threshold coincides with \citep{SN18}.
\end{example}

\begin{remark}
In Example~3 of \citep{dreveton2023exactAttributedSBM}, the authors study a semi-supervised variant of the symmetric two-community SBM. 
In our notation, this corresponds to a DBM with $k=2$, no erasures, and a noisy-label channel satisfying
\[
\mathbb{P}\bigl(U_i \neq X_i \mid X_i\bigr)=n^{-\alpha}.
\]
Based on their analysis, \cite{dreveton2023exactAttributedSBM} state that exact recovery is possible only if
\[
\lim_{n\to\infty}(\sqrt{a}-\sqrt{b})^2
-\frac{2}{\log n}\log\!\left(2\sqrt{n^{-\alpha}(1-n^{-\alpha})}\right) > 2,
\]
which simplifies to
\[
\frac{\alpha}{2}+\frac{a+b}{2}-\sqrt{ab} > 1.
\]
In contrast, the sharp threshold obtained in Example~\ref{eq:noisy_labels}, which matches the characterization in \citep{SN18}, 
does not agree with this condition. 
\end{remark}

\subsection{Two-Stage Algorithm for Exact Recovery}
In this section, we present a two-stage procedure that achieves exact recovery in the DBM with high probability and in polynomial time. The core idea is to split the observed graph into two independent subgraphs. We then use one subgraph for fast, approximate clustering and refine the clustering on the second subgraph by performing a local, per-node maximum a posteriori (MAP) estimation, which naturally incorporates both the node degree profiles and the auxiliary side information. 
Specifically, assuming that $\sigma'$ is the result of the approximate
clustering of the first stage of the algorithm, we set the degree profile
of each node
\[
\mathbf{d}(v) = \{d_1(v), d_2(v), \ldots , d_k(v) \}
\]
where $d_r(v)$ is the number of edges between node $v$ and nodes assigned by $\sigma'$ to community $r$ in the second subgraph.
We write $\mu_{s_r}^{(n)}$ for the $r$-th coordinate of the mean degree-profile vector
$\bm{\mu}_s^{(n)} = (\diag(P)\mathbf{Q})_s \log n$, i.e.,
\[
\mu_{s_r}^{(n)} = (\bm{\mu}_s^{(n)})_r = p_r Q_{rs}\log n,\qquad s,r\in[k].
\]

To execute the algorithm successfully, it is essential to resolve the inherent symmetries in the model using the observed graph structure and associated data labels. The following lemma establishes a sufficient condition under which the symmetry between any two communities can be broken with high probability, given access to $\Theta(n)$ data samples, of which $m$ may be contaminated.

\begin{lemma}[Scheff\'e test under $m$ corruptions]\label{thm:scheffe}
Let $P_X^{(1)}$ and $P_X^{(2)}$ be probability mass functions on a discrete alphabet $\mathcal{X}$.
Let $X_1,\dots,X_n$ be observations such that there exists an (unknown) index set
$\mathcal{I}\subseteq[n]$ with $|\mathcal{I}|\ge n-m$ for which the random variables
$\{X_i:i\in\mathcal{I}\}$ are i.i.d.\ with common law $P_X$, where
\[
P_X\in\{P_X^{(1)},P_X^{(2)}\},
\]
while the remaining $m$ observations $\{X_i:i\notin\mathcal{I}\}$ are arbitrary.

Consider testing
\[
H_1:\ P_X=P_X^{(1)}
\qquad\text{vs.}\qquad
H_2:\ P_X=P_X^{(2)}.
\]
Let the empirical measure be
\[
P_n(A):=\frac{1}{n}\sum_{i=1}^n \mathbbm{1}\{X_i\in A\},\qquad A\subseteq\mathcal{X},
\]
and define
\[
A^*=\arg\max_{A\subseteq\mathcal{X}} \bigl|P_X^{(1)}(A)-P_X^{(2)}(A)\bigr|.
\]
Consider the Scheff\'e decision rule
\[
\bigl|P_n(A^*)-P_X^{(1)}(A^*)\bigr|
\underset{H_1}{\overset{H_2}{\gtrless}}
\bigl|P_n(A^*)-P_X^{(2)}(A^*)\bigr|.
\]
If
\[
\mathrm{TV}\!\left(P_X^{(1)},P_X^{(2)}\right)>\frac{2m}{n},
\]
then the (prior-averaged) Bayes error probability $P_e$ of this test satisfies
\[
P_e \le 4\exp\!\left(-\frac{n}{2}\Bigl(\mathrm{TV}\!\left(P_X^{(1)},P_X^{(2)}\right)-\frac{2m}{n}\Bigr)^2\right).
\]
\end{lemma}
\begin{proof}
Let $A^*\in\arg\max_{A\subseteq\mathcal X}\bigl|P_X^{(1)}(A)-P_X^{(2)}(A)\bigr|$, so that
\[
\bigl|P_X^{(1)}(A^*)-P_X^{(2)}(A^*)\bigr|=\mathrm{TV}\!\left(P_X^{(1)},P_X^{(2)}\right).
\]
Define $S_n:=\sum_{i=1}^n \mathbbm{1}\{X_i\in A^*\}$, so that $P_n(A^*)=S_n/n$.

We first bound the type-I error under $H_1$ (i.e., $P_X=P_X^{(1)}$ on the uncorrupted samples).
An error occurs only if the empirical mass $P_n(A^*)$ is closer to $P_X^{(2)}(A^*)$ than to $P_X^{(1)}(A^*)$, i.e.
\[
\bigl|P_n(A^*)-P_X^{(1)}(A^*)\bigr|>\bigl|P_n(A^*)-P_X^{(2)}(A^*)\bigr|.
\]
By the triangle inequality,
\[
\bigl|P_X^{(2)}(A^*)-P_X^{(1)}(A^*)\bigr|
\le
\bigl|P_n(A^*)-P_X^{(1)}(A^*)\bigr|
+
\bigl|P_n(A^*)-P_X^{(2)}(A^*)\bigr|,
\]
thus on the error event we must have
\[
2\bigl|P_n(A^*)-P_X^{(1)}(A^*)\bigr|
>
\bigl|P_X^{(2)}(A^*)-P_X^{(1)}(A^*)\bigr|
=
\mathrm{TV}\!\left(P_X^{(1)},P_X^{(2)}\right).
\]
Hence
\begin{equation}\label{eq:typeI_reduce}
P_e(1)
\le
\Pr_{H_1}\!\left\{\,2\bigl|P_n(A^*)-P_X^{(1)}(A^*)\bigr|
>
\mathrm{TV}\!\left(P_X^{(1)},P_X^{(2)}\right)\right\}.
\end{equation}

Let $\mathcal I\subseteq[n]$ be the (unknown) index set of uncorrupted samples, with $|\mathcal I|\ge n-m$.
Write
\[
S_n = \sum_{i\in\mathcal I}\mathbbm{1}\{X_i\in A^*\} \;+\; \sum_{i\notin\mathcal I}\mathbbm{1}\{X_i\in A^*\}
=: S_{\mathcal I}+S_{\mathcal I^c}.
\]
Under $H_1$, the variables $\{\mathbbm{1}\{X_i\in A^*\}: i\in\mathcal I\}$ are i.i.d.\ Bernoulli with mean
$P_X^{(1)}(A^*)$, while $S_{\mathcal I^c}\in[0,m]$ is arbitrary.
Therefore,
\[
\bigl|S_n-(n-m)P_X^{(1)}(A^*)\bigr|
\le
\bigl|S_{\mathcal I}-(n-m)P_X^{(1)}(A^*)\bigr| + m,
\]
and dividing by $n$ gives
\[
\bigl|P_n(A^*)-P_X^{(1)}(A^*)\bigr|
=
\frac{1}{n}\bigl|S_n-nP_X^{(1)}(A^*)\bigr|
\le
\frac{1}{n}\bigl|S_{\mathcal I}-(n-m)P_X^{(1)}(A^*)\bigr|+\frac{m}{n}.
\]
Plugging this into \eqref{eq:typeI_reduce} yields
\[
P_e(1)
\le
\Pr_{H_1}\!\left\{
\frac{2}{n}\bigl|S_{\mathcal I}-(n-m)P_X^{(1)}(A^*)\bigr|
>
\mathrm{TV}\!\left(P_X^{(1)},P_X^{(2)}\right)-\frac{2m}{n}
\right\}.
\]
By Hoeffding's inequality for sums of $(n-m)$ i.i.d.\ Bernoulli variables,
\begin{align*}
	&\Pr_{H_1}\!\left\{
\bigl|S_{\mathcal I}-(n-m)P_X^{(1)}(A^*)\bigr|
>
\frac{n}{2}\Bigl(\mathrm{TV}(P_X^{(1)},P_X^{(2)})-\frac{2m}{n}\Bigr)
\right\}\\
&\quad\quad\quad\le
2\exp\!\left(
-\frac{1}{2}n\Bigl(\mathrm{TV}(P_X^{(1)},P_X^{(2)})-\frac{2m}{n}\Bigr)^2
\right),
\end{align*}
where we used $n-m\le n$ to simplify the exponent.
Hence,
\[
P_e(1)\le
2\exp\!\left(
-\frac{1}{2}n\Bigl(\mathrm{TV}(P_X^{(1)},P_X^{(2)})-\frac{2m}{n}\Bigr)^2
\right).
\]
The same argument under $H_2$ gives
\[
P_e(2)\le
2\exp\!\left(
-\frac{1}{2}n\Bigl(\mathrm{TV}(P_X^{(1)},P_X^{(2)})-\frac{2m}{n}\Bigr)^2
\right).
\]
Finally, for any prior on $\{H_1,H_2\}$, the Bayes error satisfies
\(
P_e \le P_e(1)+P_e(2),
\)
and therefore
\[
P_e \le
4\exp\!\left(
-\frac{1}{2}n\Bigl(\mathrm{TV}(P_X^{(1)},P_X^{(2)})-\frac{2m}{n}\Bigr)^2
\right).
\]
\end{proof}

The proposed two-stage recovery algorithm is as follows:

\vspace{1ex}
\noindent
{\sc Exact Recovery Algorithm for Data Block Model  (DBM)}\\
\noindent\textbf{Inputs:}
\begin{itemize}
    \item A graph \( g([n], \mathcal{E}) \) with vertex set \([n]\) and edge set \( \mathcal{E} \),
    \item Data labels \( L(v) \in \mathcal{U} \) for all \( v \in [n] \),
    \item DBM parameters:
    \begin{itemize}
        \item A discrete distribution \( P \) over \([k] \),
        \item An edge probability matrix \( \mathbf{Q} \in \mathbb{R}_{+}^{k \times k} \),
        \item Conditional data label distributions \( P_{U|X}(u|s) \) for \( u \in \mathcal{U}, s \in [k] \),
        \item A splitting parameter \( \gamma \in (0,1) \) and \(\delta \in (0,1)\) error parameter for {\sc Sphere-comparison} algorithm.
    \end{itemize}
\end{itemize}

\noindent\textbf{Output:} A community assignment \( \sigma'' \in [k]^n \) classifying each node \( v \in [n] \) into a community \( c \in [k] \).

\begin{enumerate}
    \item \textbf{Graph Splitting:} \\
    Independently include each edge of \( g \) in a new edge set \( \mathcal{E}' \) with probability \( \gamma \). Let \( \mathcal{E}'' = \mathcal{E} \setminus \mathcal{E}' \). Define subgraphs \( g'([n], \mathcal{E}') \) and \( g''([n], \mathcal{E}'') \).
    
    \item \textbf{Initial Community Estimation:} \label{alg:SC} \\
    Apply the \textsc{Sphere-comparison} algorithm from~\citep{abbe015} to \( g' \), yielding an approximate community assignment \( \sigma' \in [k]^n \).

        \item \textbf{Symmetry Identification:} \\
    Define the set of (approximately) indistinguishable community pairs:
    \[
    \mathcal{S} = \left\{ (s,t) \in [k] \times [k]~\middle|~
    \begin{aligned}
        & p_s = p_t, \\
        & \bm{\mu}_s = \bm{\mu}_t, \\
        & \mathrm{TV}\left(P^{(n)}_{U|X}(\cdot|s), P^{(n)}_{U|X}(\cdot|t)\right) \le \frac{2\delta}{\log n}
    \end{aligned}
    \right\}.
    \]
    Pairs in $\mathcal{S}$ correspond to communities with identical connectivity profiles, for which the attribute distributions are too close (relative to the contamination level from Step~\ref{alg:SC}) to be reliably separated in the permutation-resolution step. Accordingly, we only aim to recover labels up to relabeling within the equivalence classes induced by $\mathcal{S}$. For $(s,t)\notin \mathcal{S}$, Lemma~\ref{thm:scheffe} implies that the symmetry between $s$ and $t$ can be broken with high probability.

    \item \textbf{Refined Classification via MAP Estimation:} \\
    For each node \( v \in [n] \), compute its degree profile \( \mathbf{d}(v) \) in \( g'' \) based on the preliminary classification \( \sigma' \). Determine the final assignment \( \sigma''(v) \) using a sequence of \( k-1 \) pairwise MAP hypothesis tests:
    
    \begin{itemize}
        \item If \( (s,t) \in \mathcal{S} \), use:
        \[
        \Pr\{ \mathbf{d}(v) \mid H = s \} \cdot p_s > \Pr\{ \mathbf{d}(v) \mid H = t \} \cdot p_t \quad \Rightarrow \quad H \neq t
        \]
        
        \item If \( (s,t) \notin \mathcal{S} \), use:
        \[
        \Pr\{ \mathbf{d}(v), L(v) \mid H = s \} \cdot p_s > \Pr\{ \mathbf{d}(v), L(v) \mid H = t \} \cdot p_t \quad \Rightarrow \quad H \neq t
        \]
    \end{itemize}
    
    \item \textbf{Return:} Output the final community assignment \( \sigma'' \in [k]^n \).
\end{enumerate}

In Subsection~\ref{sub:suff} we first show that in Theorem~\ref{thm:main_ach} if $D_{s,t} > 1$ for all $s,t \in [k], s \neq t$ then {\sc Exact Recovery Algorithm for DBM} finds the correct communities (up to symmetry) with a probability of error that tends to zero as $n$ tends to infinity. Later, in subsection~\ref{sub:nec} we show that if for any $s,t \in [k], s \neq t$ we have $D_{s,t} < 1$,  then any algorithm that tries to find the communities of nodes from $(G_n,U^n)$ will have a probability of error that does not vanish.

\subsection{Proof of Sufficiency}
\label{sub:suff}
\begin{proof}[Proof of Theorem \ref{thm:main_ach}]
There is an $N_0 \in \mathbb{N}$ and  $\epsilon > 0$ 
such that
for any $n \ge N_0$ and $s,t \in [k], s \neq t$ we have
\begin{equation}
	D_{\mathrm{CT}}\left(\mathcal{P}_{\bm{\mu}_s^{(n)}}, P^{(n)}_{U|X}(\cdot|s)  \middle\|  \mathcal{P}_{\bm{\mu}_t^{(n)}}, P^{(n)}_{U|X}(\cdot|t) \right)
\ge (1 + \epsilon) \log n. \label{eq: first_ineq}
\end{equation}
For the CT-Divergence of the  distributions in Theorem~\ref{thm:main_ach}, we have
\begin{align}
& \exp\left(-D_{\mathrm{CT}}\left( \mathcal{P}_{\bm{\mu}_s^{(n)}}, P^{(n)}_{U|X}(\cdot|s)   \middle\|  \mathcal{P}_{\bm{\mu}_t^{(n)}}, P^{(n)}_{U|X}(\cdot|t) \right) \right) \nonumber \\
& \quad \quad = 
\sum_{u \in \mathcal{U}}
 \min_{\lambda_u \in [0,1]} \Biggl\{(P^{(n)}_{U|X}(u|s))^{\lambda_u}
 (P^{(n)}_{U|X}(u|t))^{1- \lambda_u} 
 \nonumber\\
& \quad\quad\quad\quad\quad \times \exp\left(
 \sum_{r=1}^k  -\lambda_u \mu_{s_r}^{(n)} - 
(1-\lambda_u) \mu_{t_r}^{(n)}  + (\mu_{s_r}^{(n)})^{\lambda_u} (\mu_{t_r}^{(n)})^{1-\lambda_u}
\right)\Biggr\},
\label{eq:lam_u^*} \\
& \quad \quad =
\sum_{u \in \mathcal{U}}
 (P^{(n)}_{U|X}(u|s))^{\lambda_u^*}  
 (P^{(n)}_{U|X}(u|t))^{1-\lambda_u^*} \nonumber\\
& \quad\quad\quad\quad\quad \times \exp\left(
 \sum_{r=1}^k  -\lambda_u^* \mu_{s_r}^{(n)} - 
(1-\lambda_u^*) \mu_{t_r}^{(n)}  + (\mu_{s_r}^{(n)})^{\lambda_u^*} (\mu_{t_r}^{(n)})^{1-\lambda_u^*}
\right),
\label{eq:lam_u^*-v2} 
\end{align}
where $\lambda_u^*$ is the optimal value that minimizes expression in \eqref{eq:lam_u^*} for each $u \in \mathcal{U}$. 
The value of $\lambda_u^*$ satisfies
\begin{align}
& \log(P^{(n)}_{U|X}(u|s))-\log(P^{(n)}_{U|X}(u|t)) \nonumber \\
& \quad\quad + \sum_{r=1}^k \Bigl(
-\mu_{s_r}^{(n)} + \mu_{t_r}^{(n)}
+ (\log(\mu_{s_r}^{(n)}) - \log(\mu_{t_r}^{(n)}))
(\mu_{s_r}^{(n)})^{\lambda_u^*} (\mu_{t_r}^{(n)})^{1-\lambda_u^*}
\Bigr)=0,
\label{eq:diff_lam}
\end{align}
which implies
\begin{align}
 & P^{(n)}_{U|X}(u|s) \prod_{r=1}^k \exp(-\mu_{s_r}^{(n)}) 
(\mu_{s_r}^{(n)})^{(\mu_{s_r}^{(n)})^{\lambda_u^*} (\mu_{t_r}^{(n)})^{1-\lambda_u^*}}  \nonumber \\
 &\quad\quad  = 
P^{(n)}_{U|X}(u|t) \prod_{r=1}^k \exp(-\mu_{t_r}^{(n)}) 
(\mu_{t_r}^{(n)})^{(\mu_{s_r}^{(n)})^{\lambda_u^*} (\mu_{t_r}^{(n)})^{1-\lambda_u^*}}.
\label{eq:eq_prob}
\end{align}
From \eqref{eq: first_ineq} and \eqref{eq:lam_u^*-v2}, we derive
\begin{align}
& \sum_{u \in \mathcal{U}}
 (P^{(n)}_{U|X}(u|s))^{\lambda_u^*}  
 (P^{(n)}_{U|X}(u|t))^{1-\lambda_u^*} \nonumber \\
&\quad\quad \times\exp\left(
 \sum_{r=1}^k  -\lambda_u^* \mu_{s_r}^{(n)} - 
(1-\lambda_u^*) \mu_{t_r}^{(n)}  + (\mu_{s_r}^{(n)})^{\lambda_u^*} (\mu_{t_r}^{(n)})^{1-\lambda_u^*}
\right)  \le  n^{-1-\epsilon}.
\label{eq:-1-e}
\end{align}
We will provide an upper bound on the error probability in the hypothesis testing problem.
\begin{align*}
   \sum_{u \in \mathcal{U}}  & \sum_{\mathbf{x} \in \mathbb{Z}^k_+} 
 	\min\left\{
	p_s P^{(n)}_{U|X}(u|s) \prod_{r=1}^k \exp(-\mu_{s_r}^{(n)} )\frac{ (\mu_{s_r}^{(n)})^{x_r}}{x_r !} , 
	p_t P^{(n)}_{U|X}(u|t) \prod_{r=1}^k \exp(-\mu_{t_r}^{(n)} )\frac{ (\mu_{t_r}^{(n)})^{x_r}}{x_r !}
	\right\} \\
& \overset{\text{(a)}}{\le}
\max\{p_s,p_t\}\sum_{u \in \mathcal{U}}  \sum_{\mathbf{x} \in \mathbb{Z}^k_+} 
(P^{(n)}_{U|X}(u|s))^{\lambda_u^*}  
 (P^{(n)}_{U|X}(u|t))^{1-\lambda_u^*} \nonumber \\
&\quad\quad \times \prod_{r=1}^k \exp( -\lambda_u^* \mu_{s_r}^{(n)} - 
(1-\lambda_u^*) \mu_{t_r}^{(n)}) 
\frac{  (\mu_{s_r}^{(n)})^{\lambda_u^* x_r} (\mu_{t_r}^{(n)})^{(1-\lambda_u^*) x_r}
 }{x_r!} \\
& =
 \max\{p_s,p_t\} \sum_{u \in \mathcal{U}}
 (P^{(n)}_{U|X}(u|s))^{\lambda_u^*}  
 (P^{(n)}_{U|X}(u|t))^{1-\lambda_u^*} \nonumber \\ 
&\quad\quad \times \exp\left(  \sum_{r=1}^k 
 -\lambda_u^* \mu_{s_r}^{(n)} - 
(1-\lambda_u^*) \mu_{t_r}^{(n)} + (\mu_{s_r}^{(n)})^{\lambda_u^*} (\mu_{t_r}^{(n)})^{1-\lambda_u^*}
 \right) \\
 & \overset{\text{(b)}}{\le}
 \max\{p_s,p_t\} n^{-1 -\epsilon },
\end{align*}
where $(a)$ follows from $\min\{x,y\} \le x^\lambda y ^{1-\lambda} , \lambda \in [0,1]$
and $(b)$ follows from \eqref{eq:-1-e}. 
 
 If the exact recovery algorithm has access to the exact degree profiles of each node, the misclassification error for each node can be bounded from above as follows:
\begin{align*}
P_e(v) & \le \frac{C}{\sqrt{(1-\gamma)\log n}}  
  \sum_{u \in \mathcal{U}}
 (P^{(n)}_{U|X}(u|s))^{\lambda_u^*}  
 (P^{(n)}_{U|X}(u|t))^{1-\lambda_u^*}  \\
&\quad\quad \times \exp\left(  (1{-}\gamma) \sum_{r=1}^k 
 -\lambda_u^* \mu_{s_r}^{(n)} {-} 
(1{-}\lambda_u^*) \mu_{t_r}^{(n)} {+} (\mu_{s_r}^{(n)})^{\lambda_u^*} (\mu_{t_r}^{(n)})^{1-\lambda_u^*}
 \right) \\
 & \le
\frac{C  \exp(  D_+(\bm{\mu}_s \| \bm{\mu}_t )\gamma \log n) }{\sqrt{(1-\gamma)\log n}}
 \sum_{u \in \mathcal{U}}
 (P^{(n)}_{U|X}(u|s))^{\lambda_u^*}  
 (P^{(n)}_{U|X}(u|t))^{1-\lambda_u^*}  \\ 
 &\quad\quad\times \exp\left(  \sum_{r=1}^k 
 -\lambda_u^* \mu_{s_r}^{(n)} - 
(1-\lambda_u^*) \mu_{t_r}^{(n)} + (\mu_{s_r}^{(n)})^{\lambda_u^*} (\mu_{t_r}^{(n)})^{1-\lambda_u^*}
 \right) \\
 & \le
 C n^{-1-\epsilon }\exp( D_+(\bm{\mu}_s \| \bm{\mu}_t) \gamma \log n).
\end{align*} 
However, the {\sc Sphere-Comparison} algorithm in Step \ref{alg:SC} of the {\sc Exact Recovery Algorithm for DBM} has a probability of error of $\delta/\log n$.
Thus, the degree profiles are distorted.
The total number of nodes with misclassified labels is given by
$\delta n (1+o(1))/ \log n  $. The probability of the degree profile of a given node based on the reported labels, divided by the probability of the degree profile based on the actual labels, can be bounded  by
\begin{align*}
  &  \frac{
    \prod_{r=1}^k e^{-\mu_{s_r}^{(n)}} (\mu_{s_r}^{(n)})^{d_r}/d_r!
    }{
    \prod_{r=1}^k e^{-\mu_{s_r}^{(n)}} (\mu_{s_r}^{(n)})^{d_r+\delta_r}/(d_r+\delta_r)!
    }
    \le
    \prod_{r=1}^k (\mu_{s_r}^{(n)})^{-\delta_r}
    \frac{
      (d_r+\delta_r)!
    }{
     d_r!
    }    \\
  &  \overset{(a)}{\le}
    \exp\Big(
    \sum_{r=1}^k -\delta_r \log(\mu_{s_r}^{(n)}) {+} \frac{1}{2}\log(1+\frac{\delta_r}{d_r}) {-} d_r \log(d_r) {-} \delta_r {+} \frac{1}{12}  {+} (d_r+\delta_r)\log(d_r+\delta_r)\Big) \\
    & \overset{(b)}{\le}
    C_1 \exp(C_2 \delta \log\log n).
\end{align*}
Here, $\delta_r$ is the number of misclassified labels in the community $r$ 
by {\sc Sphere Comparison} algorithm,
$(a)$ follows from \cite{robbins1955remark} bounds 
\[
\sqrt{2\pi n} \left(\frac{n}{e}\right)^n e^{\frac{1}{12(n+1)}} \le n! \le \sqrt{2\pi n} \left(\frac{n}{e}\right)^n e^{\frac{1}{12 n}}
\]
and $(b)$ follows from the fact that
$|\delta_r| \le \max_{i,j} Q_{i,j} \delta$ and $c_1 \log n \le d_r \le c_3 \log n$ for a large enough $c_3$ with probability $1-1/n^2$.
Thus, the Bayes hypothesis testing error with mis-classified labels is upper bounded by the Bayes hypothesis testing error with actual labels multiplied by a factor $C_1 \exp(C_2 \delta \log\log n)$. Combining this with the bound for the true degree profiles yields
\begin{align*}
  P_e(v) &\le  
  \sum_{u,\mathbf{d}: P^{(n)}_{U|X}(u|s) \mathcal{P}_{\bm{\mu}_s^{(n)}}(\mathbf{d}+\boldsymbol{\delta}) <
  P^{(n)}_{U|X}(u|t) \mathcal{P}_{\bm{\mu}_t^{(n)}}(\mathbf{d}+\boldsymbol{\delta})}
  P^{(n)}_{U|X}(u|s) \mathcal{P}_{\bm{\mu}_s^{(n)}}(\mathbf{d})
  \\
  &\quad +
  \sum_{u,\mathbf{d}: P^{(n)}_{U|X}(u|s) \mathcal{P}_{\mu_t^{(n)}}(\mathbf{d}+\boldsymbol{\delta}) >
  P^{(n)}_{U|X}(u|t) \mathcal{P}_{\mu_t^{(n)}}(\mathbf{d}+\boldsymbol{\delta})}
  P^{(n)}_{U|X}(u|t) \mathcal{P}_{\mu_t^{(n)}}(\mathbf{d})
  \\
  & \le C_1 \exp(C_2 \delta \log\log n) \hspace{-5ex}
    \sum_{u,\mathbf{d}: P^{(n)}_{U|X}(u|s) \mathcal{P}_{\mu_s^{(n)}}(\mathbf{d}+\boldsymbol{\delta}) <
  P^{(n)}_{U|X}(u|t) \mathcal{P}_{\mu_t^{(n)}}(\mathbf{d}+\boldsymbol{\delta})} \hspace{-5ex}
  P^{(n)}_{U|X}(u|s) \mathcal{P}_{\mu_s^{(n)}}(\mathbf{d}+\boldsymbol{\delta})
  \\
  &\quad +
  \sum_{u,\mathbf{d}: P^{(n)}_{U|X}(u|s) \mathcal{P}_{\mu_t^{(n)}}(\mathbf{d}+\boldsymbol{\delta}) >
  P^{(n)}_{U|X}(u|t) \mathcal{P}_{\mu_t^{(n)}}(\mathbf{d}+\boldsymbol{\delta})}
  P^{(n)}_{U|X}(u|t) \mathcal{P}_{\mu_t^{(n)}}(\mathbf{d}+\boldsymbol{\delta})
  \\
  & = C_1 \exp(C_2 \delta \log\log n)
  \sum_{u,\mathbf{d}} \min\Big\{
 P^{(n)}_{U|X}(u|s) \mathcal{P}_{\mu_s^{(n)}}(\mathbf{d}),
 P^{(n)}_{U|X}(u|t) \mathcal{P}_{\mu_t^{(n)}}(\mathbf{d})
  \Big\}
  \\
  &\le
  C_1 \exp(C_2 \delta \log\log n) C n^{-1-\epsilon}
  \exp(D_+(\boldsymbol{\mu}_s \| \boldsymbol{\mu}_t)\gamma \log n)
  \\
  &\le
  C_1 C n^{-1-\epsilon + c(\delta,\gamma)}
\end{align*}
where the vector $\boldsymbol{\delta}$ represents the changes in
true degree profiles due to misclassified nodes by {\sc Sphere-Comparison} algorithm,
and
\[
c(\delta, \gamma) = C_2 \delta \frac{\log \log n}{\log n }  + \gamma \max_{s,t} D_+(\bm{\mu}_s \| \bm{\mu}_t) .
\]
If we choose $\delta$ and $\gamma$ such that $c(\delta,\gamma) < \epsilon$, then $P_e(v) = o(1/n)$.
This completes the proof.
\end{proof}

\subsection{Converse results}
\label{sub:nec}
First, we prove that exact recovery is not solvable if the total variation between the product distribution
of data labels and degree profiles for a pair of communities is less than a certain threshold.

\begin{lemma}
\label{thm:TV}
Let $P$ be a prior distribution on $[k]$, $\mathbf{Q}\in \mathbb{R}_+^{k\times k}$ be a symmetric matrix, and $\{P^{(n)}_{U|X}\}_{n=1}^{\infty}$ be discrete measures on the set $\{\mathcal{U}^{(n)}\}_{n=1}^{\infty}$. For all integers $n\geq 1$, let \[(G_n,X^n,U^n)\sim \mathrm{DBM}\biggl(n,k,P,\frac{\log n}{n}\mathbf{Q},
	 P^{(n)}_{U|X}\biggr).\] 
Exact recovery for 
$(G_n,X^n,U^n)$
is not solvable whenever there exists  an $\epsilon > 0$
such that
\[
\mathrm{TV}\Bigl(\mathcal{P}_{\boldsymbol{\mu}^{(n)}_s} \times P_{U|X}^{(n)}(\cdot|s)  , 
               \mathcal{P}_{\boldsymbol{\mu}^{(n)}_t}  \times P_{U|X}^{(n)}(\cdot|t) \Bigr)
\le 1 - n^{-1 +\epsilon}.
\]
\end{lemma}
\begin{proof}
    Let $\mathcal{C}_r$ denote the set of nodes in the vertex set
$\mathcal V_n$ that belong to the community $r$ for $r \in [k]$.
Denote $\mathcal{H}$ to be a random subset of nodes that
belong to $\mathcal{C}_s \cup \mathcal{C}_t$ defined by choosing
every vertex  in $\mathcal{C}_s \cup \mathcal{C}_t$ independently at random with probability
$\frac{1}{\log^3 n}$. 
We consider a genie-aided problem where the community of all nodes outside the set $\mathcal{H}$ is revealed, and a given algorithm $\mathscr{A}$ determines the community of the nodes in $\mathcal{H}$.

The algorithm $\mathscr{A}$ must assign each node $v\in\mathcal{H}$, based on its observed pair $(u,\mathbf{d})\in\mathcal{U}^{(n)}\times\mathbb{Z}_+^k$, i.e., data label $u$ and degree profile $\mathbf{d}$ to either community $s$ or $t$.

Let $\mathcal{I}_s$ denote the set of data labels and degree profiles that the algorithm $\mathscr{A}$ assigns to the community $s$, and $\mathcal{I}_t$ represents the set of data labels and degree profiles that the algorithm assigns to the community $t$. Thus, the set of data labels and degree profiles partitions into $\mathcal{I}_s \sqcup \mathcal{I}_t$. 

The Bayes error of the binary hypothesis test for any algorithm $\mathscr{A}$ can be lower bounded by
\begin{align*}
    P_e(v) &\ge \min\{p_s,p_t\} \left(\sum_{\langle u, \mathbf{d}(u) \rangle \in \mathcal{I}_s} \hspace{-2ex} P_{U|X}(u|t) \mathcal{P}_{\boldsymbol{\mu}_t^{(n)}} (\mathbf{d}(u)) +  \hspace{-2ex}
                \sum_{\langle u, \mathbf{d}(u) \rangle \in \mathcal{I}_t} \hspace{-2ex} P_{U|X}(u|s) \mathcal{P}_{\boldsymbol{\mu}_s^{(n)}} (\mathbf{d}(u)) \right) \\
            &\ge \min\{p_s,p_t\} \left(\sum_{\langle u, \mathbf{d}(u) \rangle \in \mathcal{I}_s} \hspace{-2ex} P_{U|X}(u|t) \mathcal{P}_{\boldsymbol{\mu}_t^{(n)}} (\mathbf{d}(u)) + 1 {-}  \hspace{-2ex}
            \sum_{\langle u, \mathbf{d}(u) \rangle \in \mathcal{I}_s} \hspace{-3ex} P_{U|X}(u|s) \mathcal{P}_{\boldsymbol{\mu}_s^{(n)}} (\mathbf{d}(u)) \right) \\
            & \ge \min\{p_s, p_t\}(1 - \mathrm{TV}(\mathcal{P}_{\boldsymbol{\mu}_s^{(n)}} \times P_{U|X}(\cdot|s)  ,    \mathcal{P}_{\boldsymbol{\mu}_t^{(n)}} \times P_{U|X}(\cdot|t) )) \\
            & \ge \min\{p_s, p_t\} n^{-1+\epsilon}
\end{align*}
Thus, in the set $\mathcal{H}$ there are  $\Omega(n \cdot n^{-1 + \epsilon})$ vertices that are mis-classified by
the algorithm $\mathscr{A}$. We choose a set $\mathcal{W}$ of size
$\lceil \log n \rceil$ from the set of mis-classified vertices by the algorithm $\mathscr{A}$.

In the sequel we show that these $\lceil \log n \rceil$ nodes in $\mathcal{W}$
have no neighbor in $\mathcal{H}$ and thus their degree profile
can be computed exactly by knowing the community of the
nodes that have been revealed by the genie. 
Let $\hat{\mathcal{H}}_t \subseteq \mathcal{C}_t$ stand for
a fixed subset of $\mathcal{C}_t$ with cardinality $\lceil \log n \rceil$. Let $\mathcal{R}$ denote the event that
$\hat{\mathcal{H}}_t$ has no edge in $\mathcal{H}$. We have
\begin{align*}
\Pr(\mathcal{R}) & \ge 
(1-p_{\text{max}} )^{|\hat{\mathcal{H}}_t| (  |\mathcal{H}_t| - |\hat{\mathcal{H}}_t|) +{|\hat{\mathcal{H}}_t| \choose 2}} \\
&\ge
(1-p_{\text{max}})^{|\hat{\mathcal{H}}_t|   |\mathcal{H}_t| } \\
& \overset{\text{(a)}}{\ge}
\exp\left(- 2 p_{\max}   |\hat{\mathcal{H}}_t|   |\mathcal{H}_t|\right) \\
&\geq 1- 2 p_{\max}   |\hat{\mathcal{H}}_t|   |\mathcal{H}_t|\\
& \overset{\text{(b)}}{\ge}
1 -O (1/\log n),
\end{align*}
where $p_{\max} =\max_{i,j} Q_{ij} \log n /n$, and
$(a)$ follows from 
$\log(1-p_{\max}) \ge -2 p_{\max}$ for $0 \le p_{\max} \le 1/2$,
and $(b)$ follows from the fact that 
$|\mathcal{H}| = \Theta(\frac{n}{\log^3 n} )(1 \pm o_n(1))$,
$p_{\max} = O(\log n /n)$, and $|\hat{\mathcal{H}}_t| = \lceil \log n \rceil$.

Therefore, algorithm
$\mathscr{A}$ misclassifies these $\lceil \log n \rceil$ nodes in $\mathcal{W}$ and this shows the impossibility
of exact recovery for any algorithm.

\end{proof}

\begin{proof}[Proof of Theorem \ref{thm:main_conv}]
Assume that for $s,t \in [k], s \neq t$, we have $D_{s,t} < 1$. Based on \eqref{eq:conv_condition}, there exist
an $\epsilon >0$ and an infinite set of positive integers $\mathcal{N}$ such that for any $n \in \mathcal{N}$ we have
\[
D_{\mathrm{CT}}\Bigl(\mathcal{P}_{\bm{\mu}_s^{(n)}}, P^{(n)}_{U|X}(\cdot|s) \Big\|  \mathcal{P}_{\bm{\mu}_t^{(n)}}, P^{(n)}_{U|X}(\cdot|t) \Bigr)
\le (1 - \epsilon) \log n,
\]
In other words, from \eqref{eq:lam_u^*-v2}, we have
\begin{align}
&\sum_{u \in \mathcal{U}}
 (P^{(n)}_{U|X}(u|s))^{\lambda_u^*}  
 (P^{(n)}_{U|X}(u|t))^{1-\lambda_u^*} \nonumber \\
&\quad\quad \times \exp\left(
 \sum_{r=1}^k  -\lambda_u^* \mu_{s_r}^{(n)} - 
(1-\lambda_u^*) \mu_{t_r}^{(n)}  + (\mu_{s_r}^{(n)})^{\lambda_u^*} (\mu_{t_r}^{(n)})^{1-\lambda_u^*}
\right) \ge n^{-1+\epsilon}.
\label{eq:-1+ep}
\end{align}
Define the set $\mathcal{D}_{s,t}$ which contains pairs of
data labels and degree profiles as follows:
\[
\mathcal{D}_{s,t} = \left\{ \langle u, \mathbf{d}(u) \rangle \middle| 
\begin{aligned}
	&\mathbf{d}(u) = (d_1(u),d_2(u),\ldots, d_k(u)), \\
	  & d_r(u) = \lfloor (\mu_{s_r}^{(n)})^{\lambda_u^*} (\mu_{t_r}^{(n)})^{1-\lambda_u^*} \rfloor \text{ for }
	r \in [k], u \in \mathcal{U}
\end{aligned}
\right\}
\]

The probability that a node in community $s$ has a 
data label and a degree profile $\langle u, \mathbf{d}(u) \rangle$ from the set $\mathcal{D}_{s,t}$
is lower bounded as follows:
\begin{align}
\text{Pr} & (\langle u, \mathbf{d}(u) \rangle)  = 
P^{(n)}_{U|X}(u|s) \prod_{r=1}^k 
\exp(-\mu_{s_r}^{(n)}) 
\frac{(\mu_{s_r}^{(n)})^{d_r(u)}}
	 { d_r(u)!}  \nonumber \\
 & \stackrel{\text{(a)}}{\ge}
P^{(n)}_{U|X}(u|s) \prod_{r=1}^k 
\exp(-\mu_{s_r}^{(n)}) 
\frac{(\mu_{s_r}^{(n)})^{(\mu_{s_r}^{(n)})^{\lambda_u^*} (\mu_{t_r}^{(n)})^{1-\lambda_u^*}}}
{ C_1 (\log n) d_r(u)!}  \nonumber \\
& \stackrel{\text{(b)}}{=}
(P^{(n)}_{U|X}(u|s))^{\lambda_u^*} (P^{(n)}_{U|X}(u|t))^{1-\lambda_u^*} \nonumber \\
&\quad\quad \times \prod_{r=1}^k
\frac{
\exp(-\lambda_u^* \mu_{s_r}^{(n)} -  (1-\lambda_u^*)\mu_{t_r}^{(n)}) ((\mu_{s_r}^{(n)})^{\lambda_u^*} 
(\mu_{t_r}^{(n)}) ^{1-\lambda_u^*})^{(\mu_{s_r}^{(n)})^{\lambda_u^*} 
(\mu_{t_r}^{(n)}) ^{1-\lambda_u^*}}
}
{C_1 (\log n) d_r(u)!} \nonumber \\
& \stackrel{\text{(c)}}{\ge}
(P^{(n)}_{U|X}(u|s))^{\lambda_u^*} (P^{(n)}_{U|X}(u|t))^{1-\lambda_u^*} \nonumber \\
&\quad\quad \times \prod_{r=1}^k
\frac{
	\exp(-\lambda_u^* \mu_{s_r}^{(n)} -  (1-\lambda_u^*) \mu_{t_r}^{(n)}) ((\mu_{s_r}^{(n)})^{\lambda_u^*} 
	(\mu_{t_r}^{(n)})^{1-\lambda_u^*})^{(\mu_{s_r}^{(n)})^{\lambda_u^*} 
		(\mu_{t_r}^{(n)}) ^{1-\lambda_u^*}}
}
{C_2 \log n \sqrt{ (\mu_{s_r}^{(n)})^{\lambda_u^*} 
		(\mu_{t_r}^{(n)}) ^{1-\lambda_u^*} }  ((\mu_{s_r}^{(n)})^{\lambda_u^*} (\mu_{t_r}^{(n)}) ^{1-\lambda_u^*} /e )^{(\mu_{s_r}^{(n)})^{\lambda_u^*} (\mu_{t_r}^{(n)}) ^{1-\lambda_u^*} } }  \nonumber \\
& \stackrel{\text{(d)}}{=}
\frac{1}{C_2 (\log n)^{3 k/2}}
(P^{(n)}_{U|X}(u|s))^{\lambda_u^*} (P^{(n)}_{U|X}(u|t))^{1-\lambda_u^*} \nonumber \\
&\quad\quad \times \prod_{r=1}^k
\exp(-\lambda_u^* \mu_{s_r}^{(n)} -  (1-\lambda_u^*) \mu_{t_r}^{(n)} + (\mu_{s_r}^{(n)})^{\lambda_u^*} (\mu_{t_r}^{(n)}) ^{1-\lambda_u^*}  )  
\label{eq:prob_lowerbound}
\end{align}
where $(a)$ follows from the fact that $\mu_{s_r}^{(n)} = \Theta(\log n), \mu_{t_r}^{(n)} = \Theta(\log n)$ and 
\[(\mu_{s_r}^{(n)})^{\lfloor (\mu_{s_r}^{(n)})^{\lambda_u^*} (\mu_{t_r}^{(n)})^{1-\lambda_u^*} \rfloor}
\ge 
(\mu_{s_r}^{(n)})^{(\mu_{s_r}^{(n)})^{\lambda_u^*} (\mu_{t_r}^{(n)})^{1-\lambda_u^*}} / (C_1\log n),\] $(b)$ follows from \eqref{eq:eq_prob}, $(c)$
follows from Stirling's approximation $\lfloor x \rfloor! \le C \sqrt{x} (x/e)^x$, $(d)$ follows from the fact that 
$\mu_{s_r}^{(n)} = \Theta(\log n), \mu_{t_r}^{(n)} = \Theta(\log n)$.

Notice that using the same argument as above, we can show that 
the probability that a node in community $t$ has a data label and
a degree profile which is in the set $\mathcal{D}_{s,t}$ is lower bounded
by \eqref{eq:prob_lowerbound}.
Then, we can bound the total variation between degree profiles and data label distributions as follows
\begin{align*}
1- & \mathrm{TV}\Bigl(\mathcal{P}_{\boldsymbol{\mu}_s^{(n)}} \times P^{(n)}_{U|X}(\cdot|s)  , \mathcal{P}_{\boldsymbol{\mu}_t^{(n)}}\times P^{(n)}_{U|X}(\cdot |t) \Bigr) \\
& = \sum_{u\in\mathcal{U}, \mathbf{d} \in \mathbb{Z}_+^k} 
\min\Bigl\{
P^{(n)}_{U|X}(u|s) \mathcal{P}_{\boldsymbol{\mu}_s^{(n)}} (\mathbf{d}),
P^{(n)}_{U|X}(u|t) \mathcal{P}_{\boldsymbol{\mu}_t^{(n)}} (\mathbf{d})
\Bigr\} \\
& \ge \sum_{\langle u, \mathbf{d} \rangle \in \mathcal{D}_{s,t}} 
\min\Bigl\{
P^{(n)}_{U|X}(u|s) \mathcal{P}_{\boldsymbol{\mu}_s^{(n)}} (\mathbf{d}),
P^{(n)}_{U|X}(u|t) \mathcal{P}_{\boldsymbol{\mu}_t^{(n)}} (\mathbf{d})
\Bigr\} \\
& \overset{(a)}{\ge} \frac{1}{C_2 (\log n)^{3 k/2}}
\sum_{u \in \mathcal{U}}
(P^{(n)}_{U|X}(u|s))^{\lambda_u^*} (P^{(n)}_{U|X}(u|t))^{1-\lambda_u^*} \nonumber \\
&\quad\quad \times \prod_{r=1}^k
\exp\Bigl(-\lambda_u^* \mu_{s_r}^{(n)} -  (1-\lambda_u^*) \mu_{t_r}^{(n)} + (\mu_{s_r}^{(n)})^{\lambda_u^*} (\mu_{t_r}^{(n)}) ^{1-\lambda_u^*} \Bigr)
\\ & \overset{(b)}{\ge}
C_3 \frac{n^{-1+\epsilon}}{(\log n)^{3k/2}} = C_3 \exp( (-1+\epsilon) \log n - o_n(1)),
\end{align*}
where $(a)$ follows from \eqref{eq:prob_lowerbound} and $(b)$ follows from \eqref{eq:-1+ep}.
Since $(\log n)^{3k/2}=n^{o(1)}$, the bound above implies that for all sufficiently large $n$,
\[
\mathrm{TV}\Bigl(\mathcal{P}_{\boldsymbol{\mu}_s^{(n)}} \times P^{(n)}_{U|X}(\cdot|s),\ \mathcal{P}_{\boldsymbol{\mu}_t^{(n)}} \times P^{(n)}_{U|X}(\cdot|t)\Bigr)
\le 1 - n^{-1+\epsilon/2}.
\]
Therefore, Lemma~\ref{thm:TV} applies (with $\epsilon/2$), and exact recovery is not solvable.

\end{proof}

\section{Simulations}\label{sec:simulation}

In this section, we empirically study exact recovery in the logarithmic-degree regime of the two-community DBM. Latent labels are equiprobable, $P=(0.5,0.5)$, and edges are drawn independently with probabilities $\Pr\{(i,j)\in E \mid X_i=s,X_j=r\}=(\log n/n)\ Q_{sr}$, where $Q=\begin{pmatrix}a&b\\ b&a\end{pmatrix}$ with $a>b$. Side information is provided by an erased-label channel: for each vertex the true label is revealed with probability $1-n^{-\alpha}$ and is otherwise replaced by an erasure symbol (see Example~\ref{exm:erase}). In this symmetric setting the divergence that governs exact recovery is
\[
D_{1,2} \ =\  \alpha \ +\  \frac{(\sqrt a-\sqrt b)^2}{2},
\]
and Theorems~\ref{thm:main_ach} and \ref{thm:main_conv} assert that exact recovery is achievable if and only if $D_{1,2}>1$. We therefore mark on all plots the \emph{DBM threshold} $a^\star_{\rm DBM}(b,\alpha)=\big(\sqrt b+\sqrt{2(1-\alpha)}\big)^2$ and, for comparison, the \emph{SBM threshold} $a^\star_{\rm SBM}(b)=\big(\sqrt b+\sqrt2\big)^2$. 
Performance is quantified by two complementary metrics: the flip-invariant misclassification error (the fraction of misclassified vertices, minimized over global label permutations) and the exact-recovery probability (ERP), which corresponds to the notion of exact recovery in Definition~\ref{def:ER_almostER}.

Graph-based procedures (DBM, Iterative DBM, SBM, Iterative SBM, and Spectral) begin with a spectral partition of the adjacency matrix into two clusters. Except for the Spectral baseline (graph-only), we then apply a per-vertex local MAP update in the log-degree regime. For a candidate class $s$ and node $v$, the score is the sum of the log prior, the log-likelihood of the side label, and a Poisson log-likelihood for the degree profile $d(v)=(d_1(v),d_2(v))$ with means $\mu_{sr}=p_r Q_{sr}\log n$. The DBM method uses both graph and side information in this score; the SBM baseline omits the side-label term. We consider a single-pass refinement (one MAP update after the spectral start) and an iterative refinement (Algorithm~\ref{alg:iterative2}) that repeats local updates until fewer than $10^{-3}$ labels change or a cap of five iterations is reached. In the experiments reported here we use a \emph{no-split} variant, in which the full graph is used both for the spectral initialization and for MAP refinements (so the two stages are not independent). For completeness in the phase diagram, we also display the data-only MAP (side channel alone). All evaluations are flip-invariant, with independent random seeds per replicate.

\subsection{Experiment 1: Phase diagram across $(a,\alpha)$}

The goal is to localize the empirical transition to exact recovery and to visualize how side information shifts it. We fix $n=1000$ and $b=10$, sweep $a\in\{14,15,\dots,23\}$ and $\alpha\in\{0.2,0.4,0.6,0.8\}$, and for each $(a,\alpha)$ run $M=1000$ independent trials of DBM, Iterative DBM, SBM, Iterative SBM, Spectral (Graph-only), and Data-only. We show two families of cross-sections, $1-\mathrm{ERP}$ versus $a$ (linear vertical axis) and mean error versus $a$ (logarithmic vertical axis), consolidated in Figure~\ref{fig:erp-a-all} and Figure~\ref{fig:logerr-a-all}. Heatmaps over the full $(a,\alpha)$ grid appear in Figure~\ref{fig:heatmaps-accuracy} (mean accuracy) and Figure~\ref{fig:heatmaps-erp} (ERP), each with the theoretical thresholds: the DBM threshold curve
$a^\star_{\rm DBM}=(\sqrt{b}+\sqrt{2(1-\alpha)})^2$, which traces the boundary $D_{1,2}=1$ in the $(a,\alpha)$ plane, and the SBM threshold, shown as a vertical line at $a^\star_{\rm SBM}(b) = (\sqrt b+\sqrt 2)^2$.

The empirical transition follows the DBM threshold closely, with a finite-sample rightward offset that grows as the true threshold moves to smaller $a$. Using ERP as criterion, DBM reaches $95\%$ at $a=21,19,18,17$ for $\alpha=0.2,0.4,0.6,0.8$, respectively, and $99\%$ at $a=22,21,20,18$; see Tables~\ref{tab:erp95}-\ref{tab:erp99}. The $99\%$ crossings at $a=22,21,20,18$ lie roughly $+1.4$, $+0.9$, $+1.5$, and $+2.6$ to the right of $a^\star_{\rm DBM}(b,\alpha)$ (about $7\%$, $5\%$, $9\%$, and $18\%$ in relative terms), a natural consequence of $n=1000$ and the smaller expected degrees at lower $a$. By contrast, SBM, Iterative SBM, and Spectral cluster near the SBM threshold: all three require $a=22$ to exceed $95\%$ ERP and (except at $\alpha=0.2$) $a=23$ to exceed $99\%$ ERP, consistent with $a^\star_{\rm SBM}(10)\approx 20.94$ and its finite-$n$ shift.

Close inspection of the cross-sections near $a^\star_{\rm DBM}$ isolates the effect of the side channel. At $\alpha=0.4$ ($a^\star_{\rm DBM}=18.13$) the DBM ERP jumps from $\approx 0.862$ at $a=18$ to $\approx 0.952$ at $a=19$, whereas SBM is only $\approx 0.135$ and $\approx 0.459$ at the same abscissae. At $\alpha=0.8$ ($a^\star_{\rm DBM}=14.40$) DBM increases from $\approx 0.631$ at $a=14$ to $\approx 0.823$ at $a=15$, while SBM and Spectral remain near zero there. Thus, the same graph density that is subcritical for SBM becomes sufficient for exact recovery once side information is integrated into the local MAP score.

\begin{algorithm}[H]
  \caption{Iterative spectral initialization and MAP refinement (no-split variant used in the simulations)}
  \label{alg:iterative2}
  \begin{algorithmic}[1]
    \Require
      Graph $G=(V,E)$ with $|V|=n$;
      side observations $\{U_v\}_{v\in V}$ with channel $P_{U\mid X}$;
      class prior $P=(p_1,\dots,p_k)$;
      affinity matrix $Q\in\mathbb{R}_{>0}^{k\times k}$;
      tolerance $\varepsilon>0$;
      maximum MAP iterations $T_{\max}\in\mathbb{N}$.
    \Ensure Estimated labels $\widehat\sigma\in[k]^n$.

    \State \textbf{Spectral initialization}
    \State $\sigma' \gets \mathrm{SpectralClustering}(G, k)$.

    \State \textbf{Label canonicalization (permutation resolution)}
    \State Define the set of informative anchors
      \[
        S \gets \Bigl\{ v\in V:\ \exists  s\neq t\ \text{with}\ P_{U\mid X}(U_v\mid s)\neq P_{U\mid X}(U_v\mid t)\Bigr\}.
      \]
    \If{$S\neq\varnothing$} \Comment{canonicalize using side information}
       \State Form $C\in\mathbb{R}^{k\times k}$ with
         \[
           C_{i,t} \gets \sum_{\substack{v\in S:\\ \sigma'(v)=i}} \log P_{U\mid X}(U_v \mid t).
         \]
       \State Compute $\pi\in S_k$ maximizing $\sum_{i=1}^k C_{i,\pi(i)}$ (e.g., Hungarian on $-C$).
    \Else \Comment{no informative anchors: canonicalize using $G$}
       \State For $i\in[k]$, let $m_i \gets |\{v:\sigma'(v)=i\}|$;\quad
              for $1\le i\le j\le k$, let
              $E_{ij}\gets |\{(u,v)\in E:\sigma'(u)=i, \sigma'(v)=j, u<v\}|$.
       \State Select any deterministic maximizer
         \[
           \pi \in \arg\max_{\pi\in S_k}
           \Biggl[
             \sum_{i=1}^k m_i \log p_{\pi(i)}
             + \sum_{1\le i\le j\le k} E_{ij} \log Q_{\pi(i), \pi(j)}
           \Biggr].
         \]
    \EndIf
    \State Relabel $\sigma' \gets \pi\circ\sigma'$ and set $\widehat\sigma^{(0)}\gets\sigma'$.

    \State \textbf{MAP refinement on $G$}
    \State Precompute $\mu_{sr} \gets p_r Q_{sr} \log n$ for all $s,r\in[k]$.
    \For{$t=1$ \textbf{to} $T_{\max}$}
      \ForAll{$v\in V$}
        \State For each $r\in[k]$, compute
               $d_r(v) \gets |\{ u:\ (v,u)\in E,\ \widehat\sigma^{(t-1)}(u)=r \}|$.
        \State For each $s\in[k]$, set
          \[
            \mathsf{s}_v(s)
            =
            \log p_s
            + \log P_{U\mid X}(U_v\mid s)
            + \sum_{r=1}^k \log \mathrm{Poisson}\bigl(d_r(v);\mu_{sr}\bigr).
          \]
        \State $\widehat\sigma^{(t)}(v) \gets \arg\max_{s\in[k]} \mathsf{s}_v(s)$.
      \EndFor
      \State $\delta \gets \frac{1}{n}\sum_{v\in V}\mathbf{1}\{\widehat\sigma^{(t)}(v)\neq \widehat\sigma^{(t-1)}(v)\}$.
      \If{$\delta<\varepsilon$} \textbf{break} \EndIf
    \EndFor

    \State \Return $\widehat\sigma^{(t)}$.
  \end{algorithmic}
\end{algorithm}

The log-error panels clarify why mean errors can appear tiny before ERP saturates. At $\alpha=0.8$ and $a=14$ the DBM accuracy already exceeds $0.999$, yet ERP is only $\approx 0.63$. This illustrates the expected gap between vanishing \emph{average} error (a handful of mistakes) and the stricter event of \emph{no} mistakes (exact recovery): at $n=1000$, even a single error contributes only $10^{-3}$ to the mean error but prevents exact recovery. Near the transition, each unit increase in $a$ typically multiplies the DBM mean error by about $2.5$--$5$ (a drop of $0.4$--$0.7$ decades per unit on the log scale), until curves flatten at the Monte Carlo floor.

Iteration via Algorithm~\ref{alg:iterative2} has negligible effect on mean error but yields a modest and meaningful gain in ERP near the transition when side information is abundant. Averaged across the two $a$-ticks bracketing $a^\star_{\rm DBM}(b,\alpha)$, the increase $\mathrm{ERP}_{\rm iter}-\mathrm{ERP}_{\rm one}$ is $\approx 0.0025$ at $\alpha\in\{0.2,0.4\}$, rises to $\approx 0.023$ at $\alpha=0.6$, and reaches $\approx 0.10$ at $\alpha=0.8$, matching the intuition that a second scoring pass with cleaner neighborhood assignments flips the last stubborn vertices. Baselines behave as expected: Data-only ERP is independent of $a$ and matches the erased-label formula $\mathbb{E}[2^{-K}]=(1-n^{-\alpha}/2)^n$ (e.g., $\approx 0.136$ at $n=1000$, $\alpha=0.8$), while Spectral tracks SBM and approaches one around the SBM threshold.

Runtime trends complement the statistical picture. Figure~\ref{fig:runtime-a-08} shows wall-clock time versus $a$ at $\alpha=0.8$ (the other $\alpha$ display the same pattern). Single-pass DBM and the SBM Poisson scorer are essentially flat near $0.55$\ s per replicate. Iterative DBM and Iterative SBM cost $2$-$3$\ s at small $a$ and decrease to $\approx 0.56$\ s as $a$ increases, reflecting fewer iterations as the problem eases; the average number of iterations for Iterative DBM drops from about $2.25$ at $\alpha=0.2$ to $1.62$ at $\alpha=0.8$, with Iterative SBM around $2.6$ across $\alpha$. Spectral is consistently fast ($0.15$-$0.20$\ s), and Data-only is essentially instantaneous.

\begin{figure}[t]
  \centering
  \begin{subfigure}[t]{.49\linewidth}
    \includegraphics[width=\linewidth]{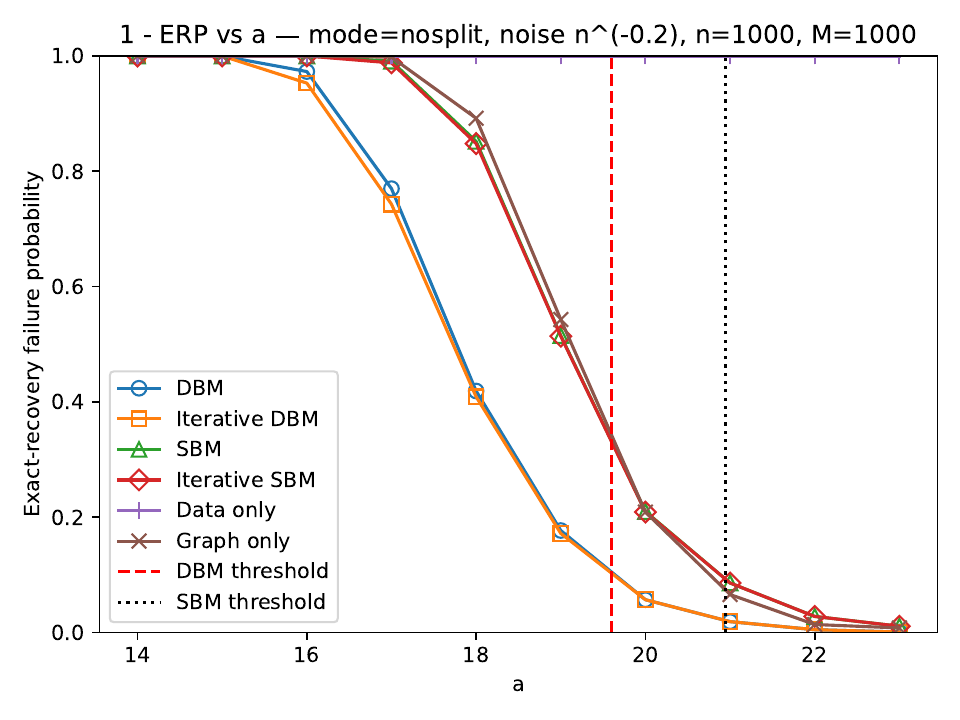}\caption{$\alpha=0.2$}
  \end{subfigure}\hfill
  \begin{subfigure}[t]{.49\linewidth}
    \includegraphics[width=\linewidth]{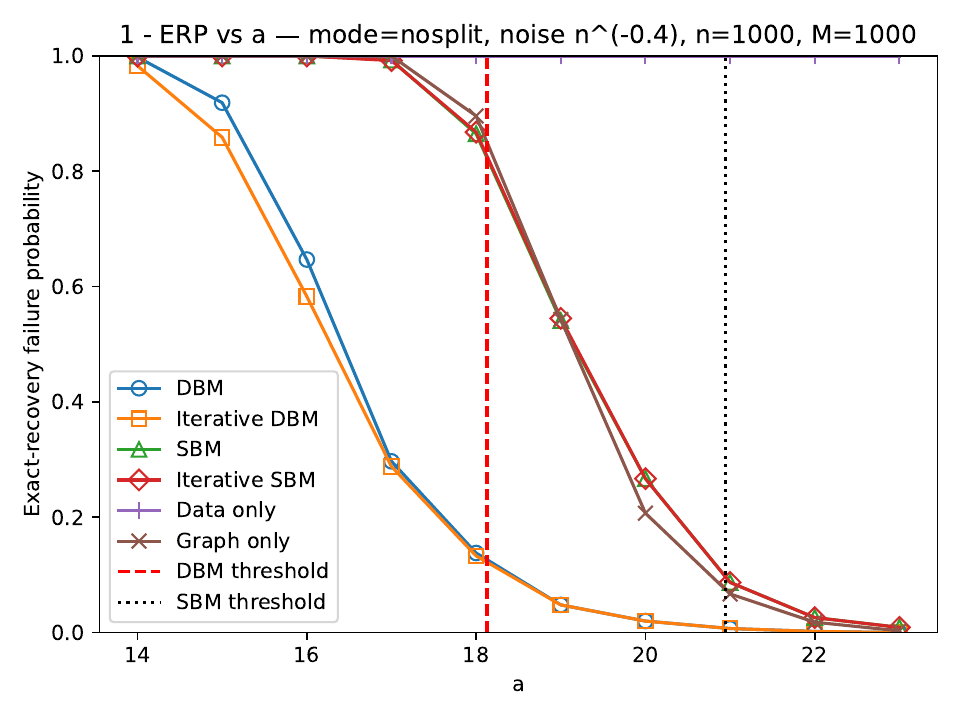}\caption{$\alpha=0.4$}
  \end{subfigure}\\[0.5em]
  \begin{subfigure}[t]{.49\linewidth}
    \includegraphics[width=\linewidth]{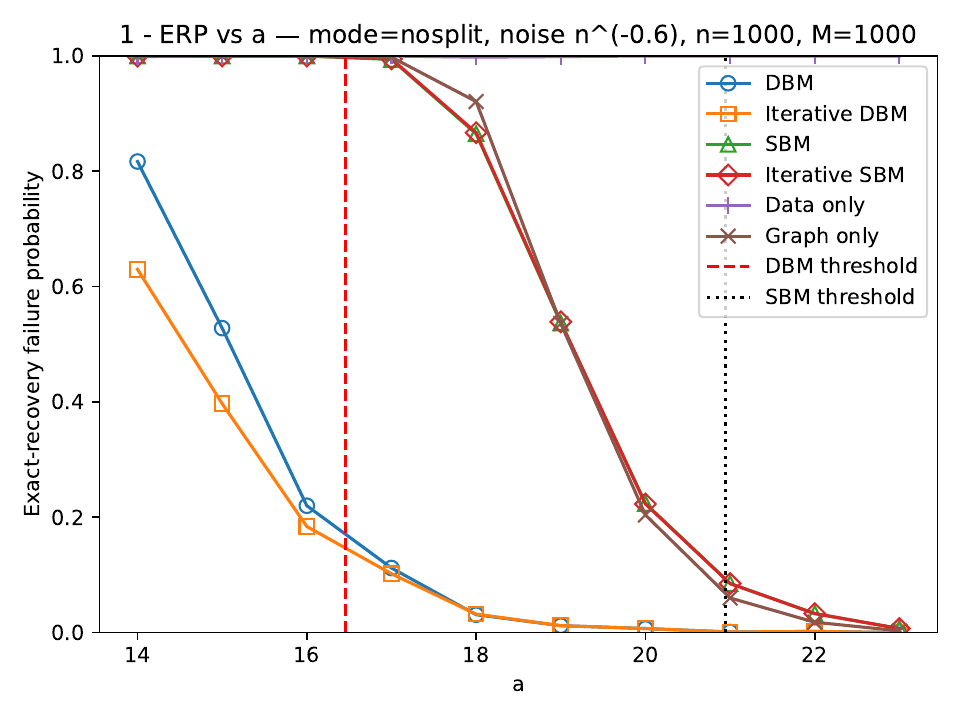}\caption{$\alpha=0.6$}
  \end{subfigure}\hfill
  \begin{subfigure}[t]{.49\linewidth}
    \includegraphics[width=\linewidth]{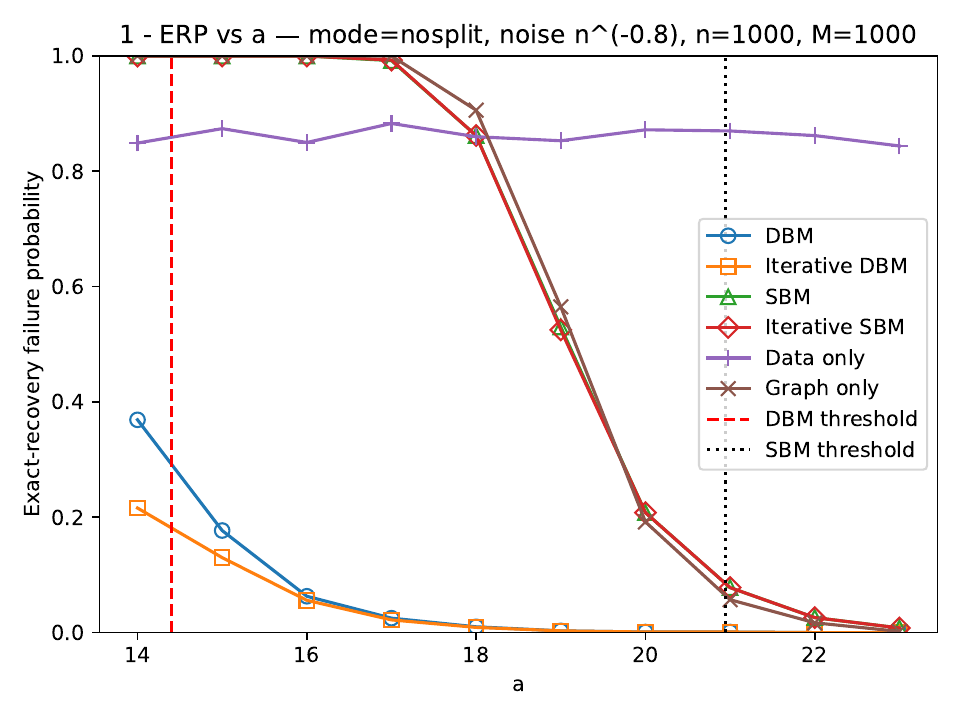}\caption{$\alpha=0.8$}
  \end{subfigure}
  \caption{$1-\mathrm{ERP}$ versus $a$ at fixed $\alpha$ ($n=1000$, $b=10$, $M=1000$). Vertical guidelines mark $a^\star_{\rm DBM}(b,\alpha)$ and $a^\star_{\rm SBM}(b)$.}
  \label{fig:erp-a-all}
\end{figure}

\begin{figure}[t]
  \centering
  \begin{subfigure}[t]{.49\linewidth}
    \includegraphics[width=\linewidth]{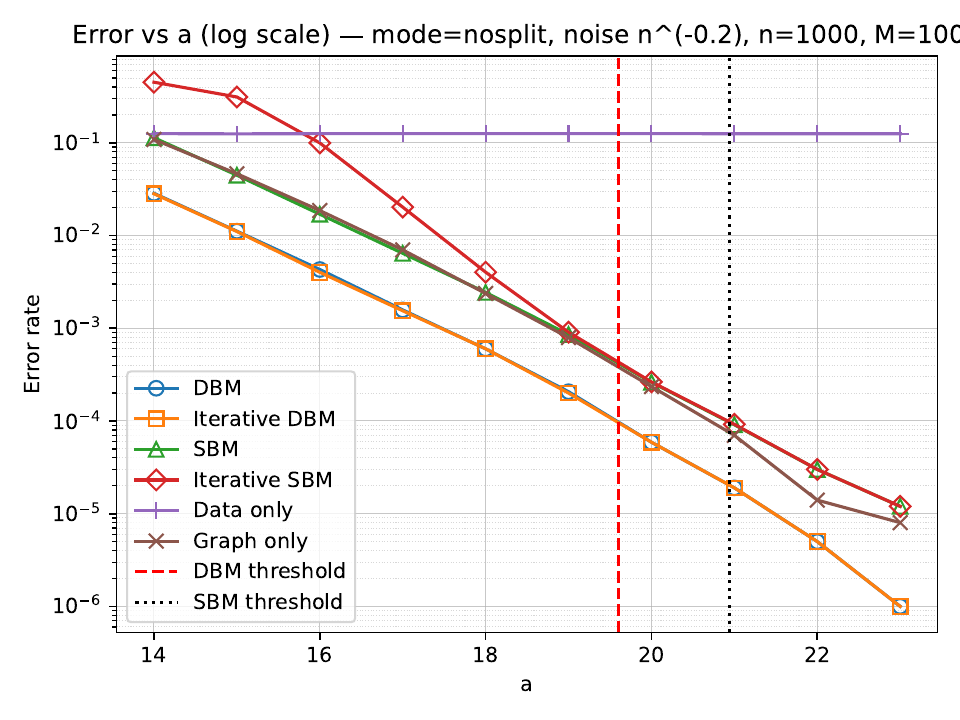}\caption{$\alpha=0.2$}
  \end{subfigure}\hfill
  \begin{subfigure}[t]{.49\linewidth}
    \includegraphics[width=\linewidth]{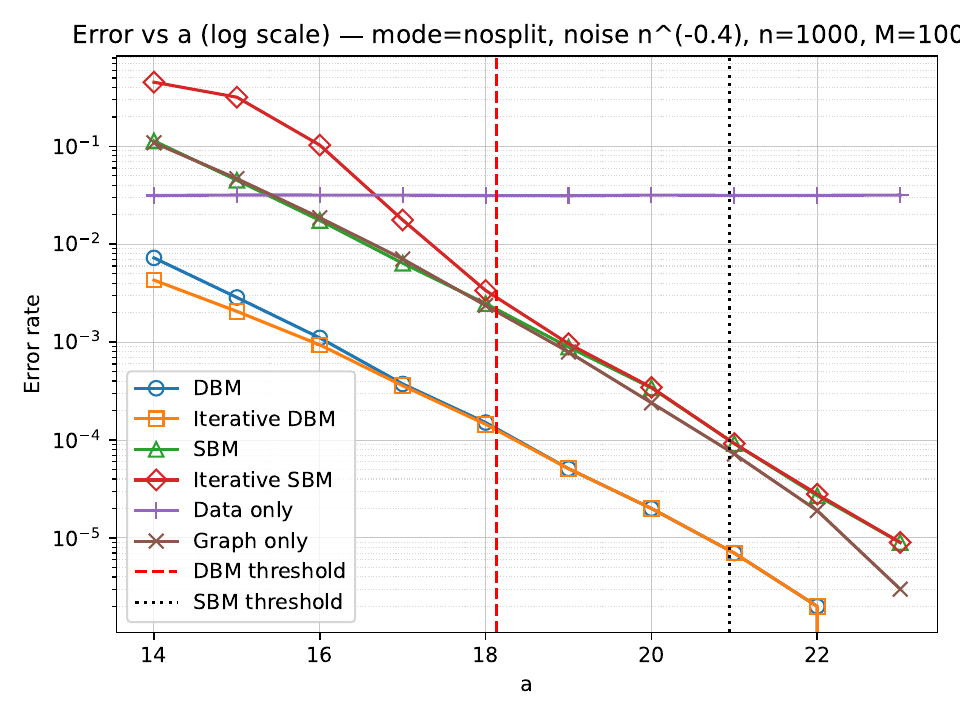}\caption{$\alpha=0.4$}
  \end{subfigure}\\[0.5em]
  \begin{subfigure}[t]{.49\linewidth}
    \includegraphics[width=\linewidth]{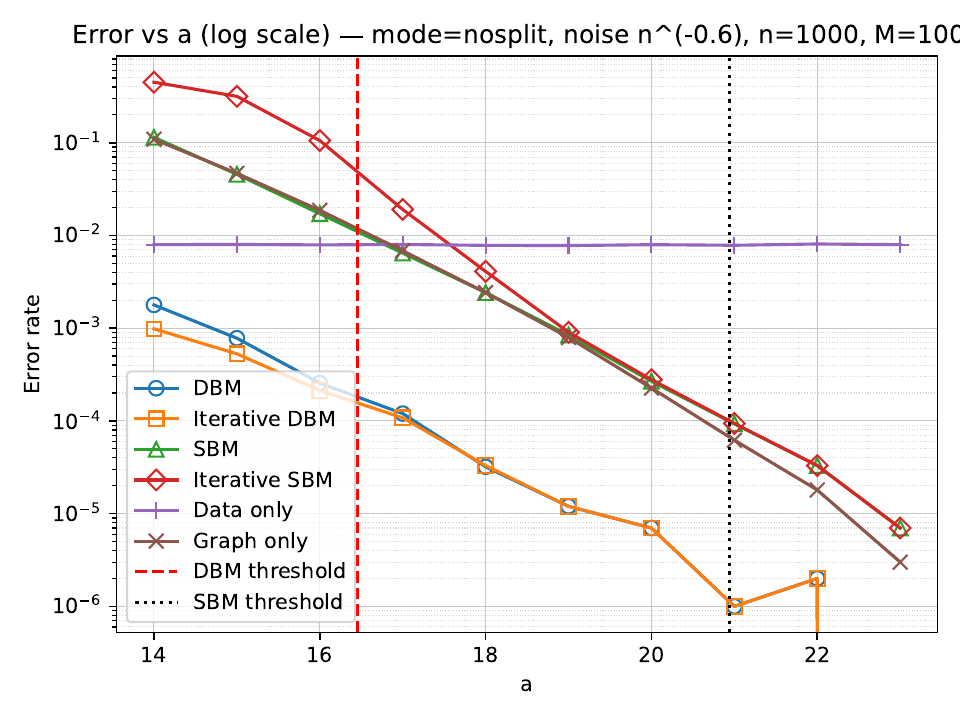}\caption{$\alpha=0.6$}
  \end{subfigure}\hfill
  \begin{subfigure}[t]{.49\linewidth}
    \includegraphics[width=\linewidth]{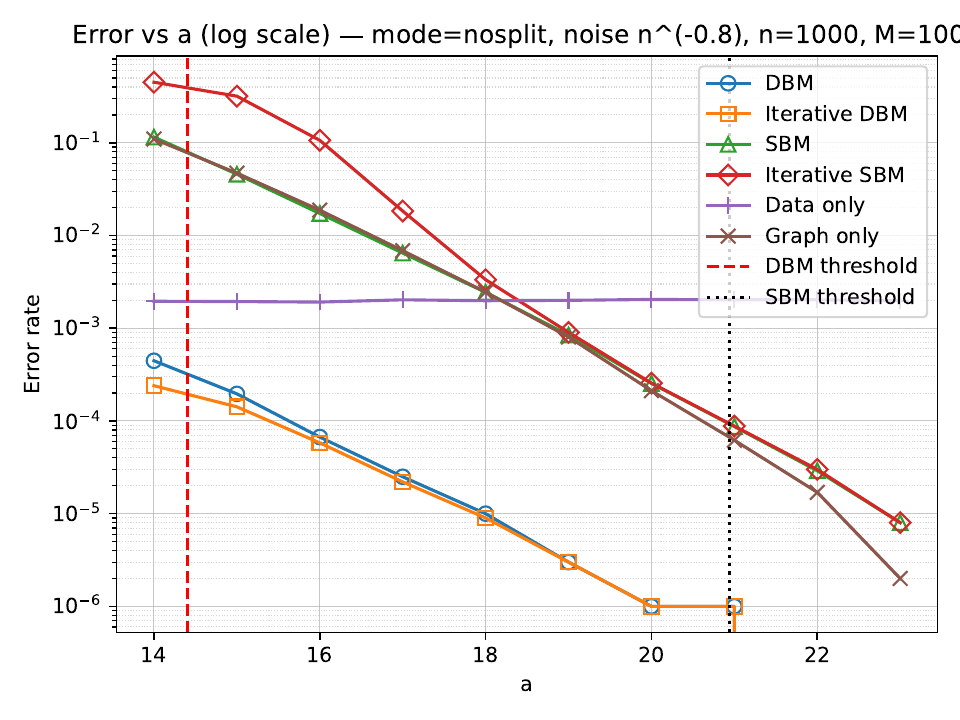}\caption{$\alpha=0.8$}
  \end{subfigure}
  \caption{Mean misclassification error (log scale) versus $a$ at fixed $\alpha$ ($n=1000$, $b=10$, $M=1000$).}
  \label{fig:logerr-a-all}
\end{figure}

\begin{figure}[t]
  \centering
  \includegraphics[width=\linewidth]{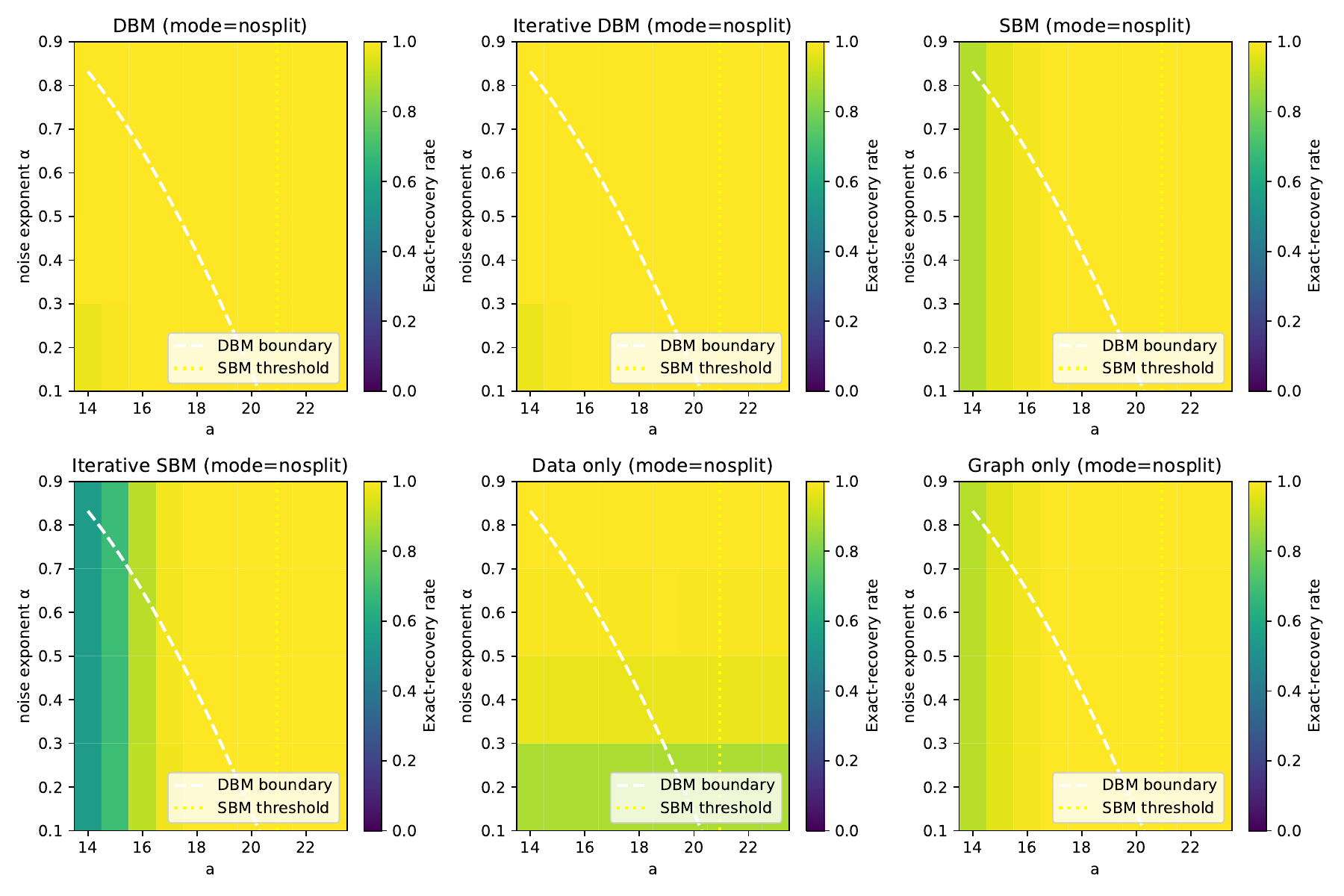}
  \caption{Heatmaps of mean accuracy across the $(a,\alpha)$ grid. The DBM threshold curve $a^\star_{\rm DBM}(b,\alpha)=(\sqrt b+\sqrt{2(1-\alpha)})^2$ and the SBM threshold at $a^\star_{\rm SBM}(b)=(\sqrt b+\sqrt2)^2$ are overlaid.}
  \label{fig:heatmaps-accuracy}
\end{figure}

\begin{figure}[t]
  \centering
  \includegraphics[width=\linewidth]{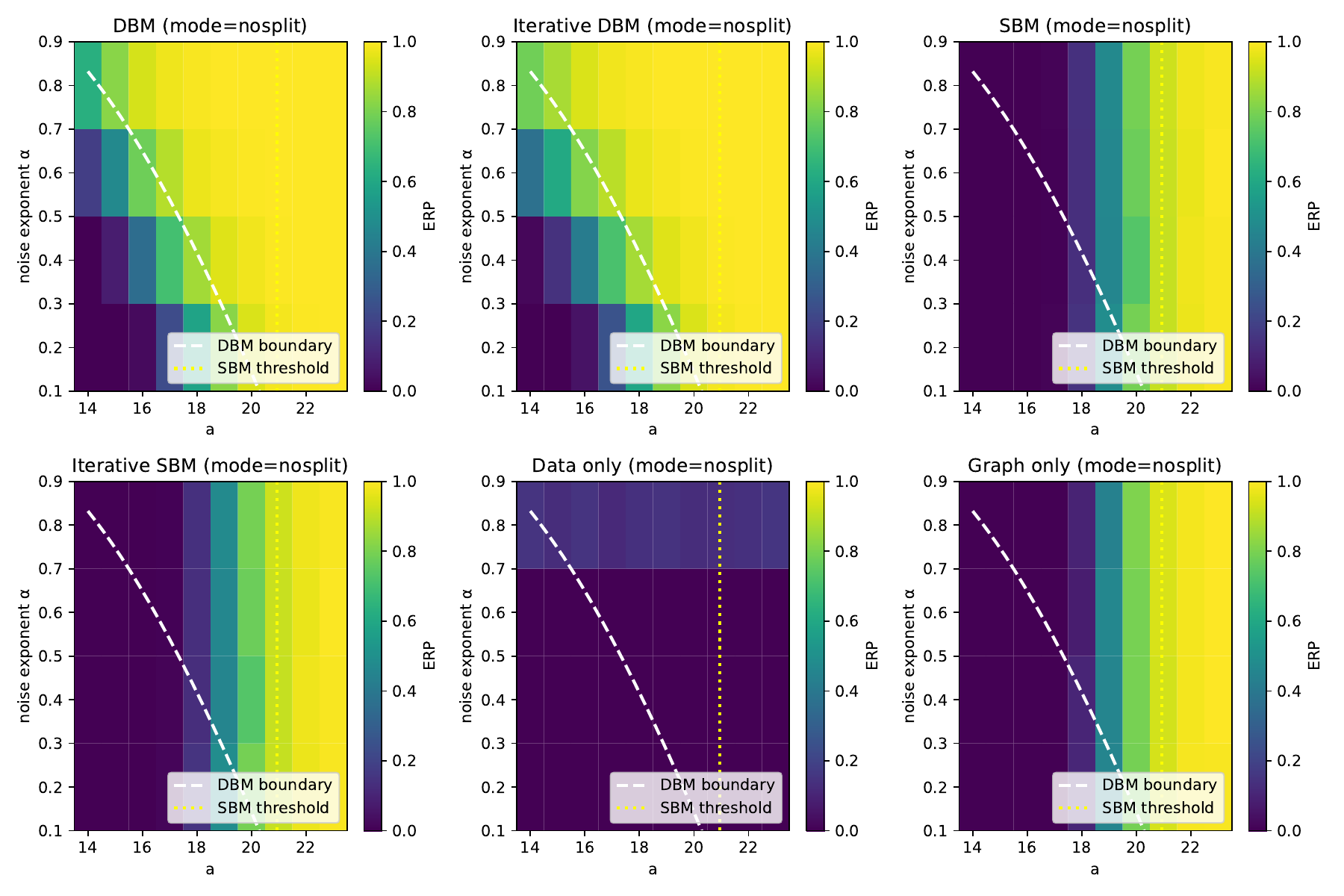}
  \caption{Heatmaps of ERP across the $(a,\alpha)$ grid, with the same threshold overlays.}
  \label{fig:heatmaps-erp}
\end{figure}

\begin{figure}[t]
  \centering
  \includegraphics[width=.5\linewidth]{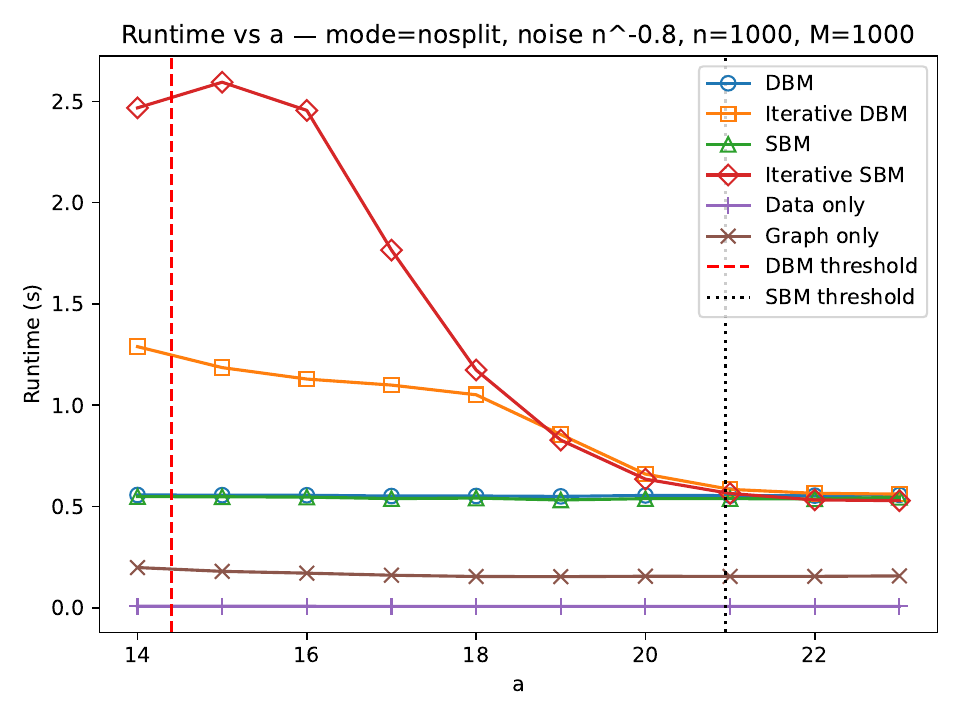}
  \caption{Runtime versus $a$ at $\alpha=0.8$ ($n=1000$, $b=10$, $M=1000$). Iterative methods slow down at small $a$ and approach single-pass runtimes as $a$ increases; Spectral and Data-only are fastest. The other $\alpha$ shows the same qualitative pattern.}
  \label{fig:runtime-a-08}
\end{figure}

\begin{table}[t]
\centering
\small
\begin{tabular}{lcccc}
\toprule
Method & $\alpha=0.2$ & $\alpha=0.4$ & $\alpha=0.6$ & $\alpha=0.8$ \\
\midrule
DBM & 21 & 19 & 18 & 17 \\
Iterative DBM & 21 & 19 & 18 & 17 \\
SBM & 22 & 22 & 22 & 22 \\
Iterative SBM & 22 & 22 & 22 & 22 \\
Spectral & 22 & 22 & 22 & 22 \\
Data-only & -- & -- & -- & -- \\
\bottomrule
\end{tabular}
\caption{Smallest integer $a$ at which ERP $\ge 95\%$ for each method as a function of $\alpha$ ($n=1000$, $b=10$, $M=1000$). `--' indicates the level is not reached within $a\in\{14,\dots,23\}$.}
\label{tab:erp95}
\end{table}

\begin{table}[t]
\centering
\small
\begin{tabular}{lcccc}
\toprule
Method & $\alpha=0.2$ & $\alpha=0.4$ & $\alpha=0.6$ & $\alpha=0.8$ \\
\midrule
DBM & 22 & 21 & 20 & 18 \\
Iterative DBM & 22 & 21 & 20 & 18 \\
SBM & -- & 23 & 23 & 23 \\
Iterative SBM & -- & 23 & 23 & 23 \\
Spectral & 23 & 23 & 23 & 23 \\
Data-only & -- & -- & -- & -- \\
\bottomrule
\end{tabular}
\caption{Smallest integer $a$ at which ERP $\ge 99\%$ for each method as a function of $\alpha$ ($n=1000$, $b=10$, $M=1000$). `--' indicates the level is not reached within $a\in\{14,\dots,23\}$.}
\label{tab:erp99}
\end{table}

\subsection{Experiment 2: Finite-size scaling in $n$}

We next examine how quickly finite problems inherit the asymptotic picture as $n$ grows. To isolate the effect, we set $(b,\alpha)=(10,0.3)$ and choose
\[
a \ =\  1.10\  a^\star_{\rm DBM}(b,\alpha)
\ =\  1.10\big(\sqrt{10}+\sqrt{2(1-0.3)}\big)^2
\ \approx\  20.77,
\]
which is supercritical for DBM and strictly below the SBM threshold $a^\star_{\rm SBM}(10)\approx 20.94$. We vary $n\in\{10,10^2,10^3,10^4\}$ and for each $n$ run $M=1000$ independent trials of DBM, Iterative DBM, SBM, and Iterative SBM. Figure~\ref{fig:scaling-erp-error}a shows the failure probability $1-\mathrm{ERP}$ versus $n$; Figure~\ref{fig:scaling-erp-error}b shows mean error versus $n$ on log-log axes; a numerical summary appears in Table~\ref{tab:scaling-summary}.

The trends are the finite-sample signatures of the thresholds. Because $a$ is chosen $10\%$ above the DBM threshold and below the SBM threshold, $1-\mathrm{ERP}$ for DBM decays with $n$, while $1-\mathrm{ERP}$ for SBM increases. By $n=10^4$, DBM attains $\mathrm{ERP}\approx 0.992$ whereas SBM declines to $\mathrm{ERP}\approx 0.892$. Meanwhile, typical errors shrink rapidly: the DBM mean error drops by orders of magnitude from $n=10$ to $n=10^2$ and is indistinguishable from zero for $n\ge 10^3$, closely paralleling a $1/n$ benchmark; the SBM average error also decays despite being subcritical for exact recovery, underscoring the difference between vanishing mean error and the event of zero mistakes. Iteration (Algorithm~\ref{alg:iterative2}) has little effect on mean error once the spectral start is good, but marginally improves ERP at larger $n$ by flipping residual errors.

Runtime scales near-linearly with $n$, rising from milliseconds at $n=10$ to about eight seconds at $n=10^4$ per replicate. The iterative variants are a constant factor slower than single-pass at small sizes, but the gap closes with $n$ because Algorithm~\ref{alg:iterative2} typically stabilizes in a single iteration: the average iteration counts fall from about $2.2$ at $n=10$ to $1.0$ at $n=10^4$. In large problems, the statistical benefit of iteration thus comes at negligible additional cost.

\begin{figure}[t]
  \centering
  \begin{subfigure}[t]{.49\linewidth}
    \includegraphics[width=\linewidth]{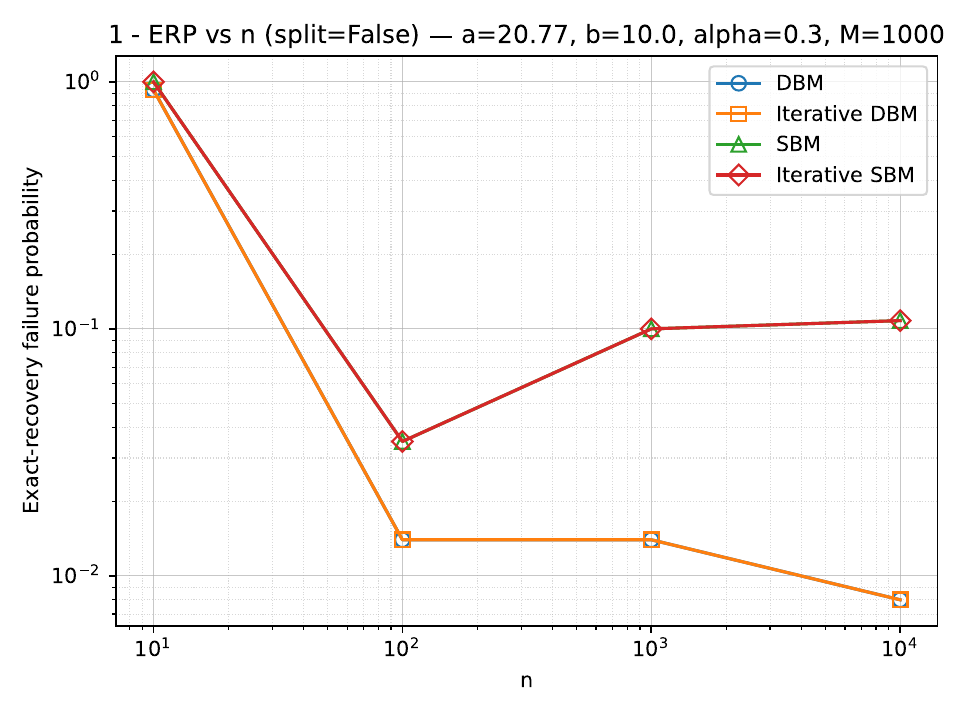}
    \caption{$1-\mathrm{ERP}$ versus $n$ at $(a,b,\alpha)=(20.77,10,0.3)$.}
  \end{subfigure}\hfill
  \begin{subfigure}[t]{.49\linewidth}
    \includegraphics[width=\linewidth]{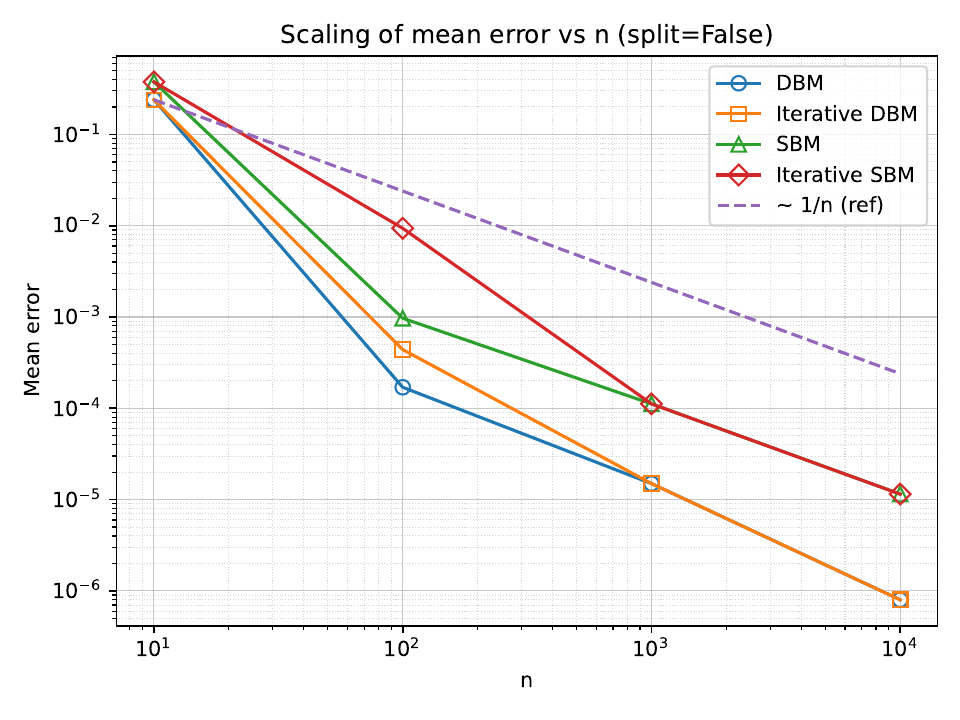}
    \caption{Mean error versus $n$ (log-log) at the same parameters; the thin line marks slope $-1$.}
  \end{subfigure}
  \caption{Finite-size scaling. DBM is supercritical, and its failure probability decays with $n$, while SBM is subcritical and its failure probability increases. Typical errors decay rapidly for both, but only DBM exhibits growing exact-recovery probability.}
  \label{fig:scaling-erp-error}
\end{figure}

\begin{table}[t]
\centering
\small
\begin{tabular}{rcccccccc}
\toprule
$n$ & err$_{\rm DBM}$ & err$_{\rm SBM}$ & ERP$_{\rm DBM}$ & ERP$_{\rm SBM}$ & time$_{\rm DBM}$ & time$_{\rm SBM}$ & iters$_{\rm iDBM}$ & iters$_{\rm iSBM}$ \\
\midrule
$10$   & $0.240$ & $0.376$ & $0.071$ & $0.000$ & $0.004$ & $0.004$ & $2.20$ & $2.42$ \\
$10^2$ & $4\times 10^{-4}$ & $9\times 10^{-3}$ & $0.986$ & $0.965$ & $0.046$ & $0.046$ & $1.03$ & $1.10$ \\
$10^3$ & $0$ & $10^{-4}$ & $0.986$ & $0.900$ & $0.532$ & $0.521$ & $1.08$ & $1.10$ \\
$10^4$ & $0$ & $0$ & $0.992$ & $0.892$ & $7.24$ & $7.83$ & $1.00$ & $1.00$ \\
\bottomrule
\end{tabular}
\caption{Scaling at $(a,b,\alpha)=(20.77,10,0.3)$, averaged over $M=1000$ replicates per $n$. DBM lies above its threshold and SBM below it. Entries report mean error, exact-recovery probability, runtime (seconds per replicate), and average iteration counts for the iterative variants.}
\label{tab:scaling-summary}
\end{table}

\section{Conclusion}\label{sec:conclusion}

In this work, we studied 
{exact recovery in the \emph{Data Block Model} (DBM), an extension of the stochastic block model that couples graph connectivity with node-associated side information.} Building on the Chernoff--Hellinger divergence framework { for SBMs in the logarithmic-degree regime}, we introduced the Chernoff--TV divergence and showed that it precisely characterizes the phase transition for exact recovery in the DBM under logarithmic-degree scaling of edges. We established matching achievability and converse results and provided an efficient { polynomial-time (in fact, near-linear-time) } algorithm that attains the information-theoretic threshold. {Simulations in the symmetric two-community setting corroborate the predicted phase transition and illustrate the benefits of incorporating node attributes.}

Our analysis clarifies when and how side information fundamentally improves recoverability { beyond what is possible from the graph alone.}.  
These results unify and refine several previously studied { side-information} models, including noisy labels, partially revealed labels, and node features, under a common divergence-based characterization. In a restricted {node-attributed SBM} setting, our results also revisit a recent sufficient condition for exact recovery and show that it is not necessary as stated, providing a corrected {sharp} threshold based on the Chernoff--TV divergence.

\section*{Acknowledgments}
Amir R. Asadi is supported by Leverhulme Trust grant ECF-2023-189 and
Isaac Newton Trust grant 23.08(b).

\clearpage
\bibliographystyle{plainnat}
\bibliography{SBM-side}

\end{document}